\documentclass[a4paper,11pt,reqno,nocut]{article} 
\usepackage[utf8]{inputenc}
\usepackage[T1]{fontenc}
\usepackage{amsmath,amssymb,mathtools}
\usepackage{mathrsfs,bm,color}
\usepackage{fancyhdr,lastpage}
\usepackage[shortlabels]{enumitem}
\usepackage{calc}
\usepackage{bbm}
\usepackage{graphicx}
\usepackage{upgreek}
\usepackage{lmodern}
\usepackage[normalem]{ulem}
\usepackage{verbatim}
\usepackage{bbm}
\usepackage{stmaryrd}
\usepackage{amsmath}
\usepackage{amssymb}
\usepackage{dsfont}
\usepackage{amsthm}
\usepackage{hyperref}
\hypersetup{
	colorlinks,
	linkcolor={red!50!black},
	citecolor={red!80!black},
	urlcolor={blue!80!black}
}

\usepackage{xcolor}

\usepackage{graphicx}

\usepackage{amsfonts}%
\usepackage{dsfont}

\newcommand{\Var}{\operatorname{Var}}

\allowdisplaybreaks
\usepackage[hmargin=2cm,vmargin=2cm]{geometry}

\numberwithin{equation}{section}

\usepackage{thmtools}
\declaretheorem[thmbox=M,name=Theorem,numberwithin=section]{theo}
\declaretheorem[name=Proposition,thmbox=M,numberwithin=section]{proposition}
\declaretheorem[name=Lemma,thmbox=S,numberwithin=section]{lem}
\declaretheorem[name=Corollary,thmbox=M,numberwithin=section]{cor}

\declaretheorem[name=Hypothesis,thmbox=S,numberwithin=section]{hyp}
\theoremstyle{definition}

\newtheorem{definition}{Definition}[section]
\newtheorem{remark}{Remark}[section]
\theoremstyle{plain}

\makeatletter
\renewenvironment{proof}[1][\proofname]{\par
	\pushQED{\qed}%
	\normalfont \topsep6\p@\@plus6\p@\relax
	\trivlist
	\item[\hskip\labelsep
	\sffamily\bfseries
	#1\@addpunct{.}]\ignorespaces
}{%
	\popQED\endtrivlist\@endpefalse
}
\makeatother

\newcommand{\jump}{{\vskip 0.3cm \noindent }}

\renewcommand{\and}{\mbox{ and }}

\newcommand{\N}{\mathbb{N}}

\newcommand{\R}{\mathbb{R}}
\newcommand{\C}{\mathbb{C}}

\newcommand{\dd}{\mathrm d}

\newcommand{\Ind}[1]{\1_{(#1)}}

\newcommand{\E}{\mathbb{E}}
\renewcommand{\Var}{\mathbb{V}}

\def\as{{ \mathrm{a.s.}  }}
\def\eas{\stackrel{{\as }}{=}}
\def\ed{\stackrel{{\mathcal{D}}}{=}}
\def\toed{\overset{\mathcal{D}}{\Rightarrow}}
\def\iid{\stackrel{{\mathrm{i.i.d}}}{\sim}}

\newcommand{\vect}[1]{\ensuremath{\boldsymbol{\mathbf{#1}}}}
\newcommand{\rdmvect}[1]{\ensuremath{\bm{#1}}}
\newcommand{\mat}[1]{\ensuremath{\boldsymbol{\mathbf{#1}}}}
\newcommand{\rdmmat}[1]{\ensuremath{\bm{#1}}}

\newcommand{\Tr}{\ensuremath{\mathrm{Tr}}}

\newcommand{\normop}[1]{\ensuremath{ \|#1 \|_{\mathrm{op}} }}

\newcommand\given[1][]{\:#1\vert\:}
\DeclareMathOperator{\1}{\mathbbm{1}}

\newcommand{\cH}{\mathcal{H}}

\renewcommand{\leq}{\leqslant}
\renewcommand{\geq}{\geqslant}
\renewcommand{\epsilon}{\varepsilon}

\def\pP{{\mathbb P}}

\def\R{{\mathbb R}}

\def\prior{{\pi}}
\def\pout{{ p_{\mathrm{out}} }}

\def\dprior{{\dd \pi}}
\def\dpriortensor{{\dprior^{\otimes N}}}
\def\priortensor{{\prior^{\otimes N}}}
\def\dpriork{{\dd \pi_{k_F}}}
\def\kf{{k_F}}

\begin{document}
	\title{Fundamental limits of Non-Linear Low-Rank Matrix Estimation}
	
	\author{
		Pierre Mergny\thanks{IdePHICS laboratory, \'Ecole F\'ed\'erale Polytechnique de Lausanne, Switzerland. Email: \texttt{pierre.mergny@epfl.ch}},
		Justin Ko\thanks{UMPA, ENS Lyon and CNRS, France and Department of Statistics and Actuarial Science, University of Waterloo, Canada. Email: \texttt{justin.ko@uwaterloo.ca}},
		Florent Krzakala\thanks{IdePHICS laboratory, \'Ecole F\'ed\'erale Polytechnique de Lausanne, Switzerland. Email: \texttt{florent.krzakala@epfl.ch}}
		,
		Lenka Zdeborov\'a\thanks{SPOC laboratory,  \'Ecole F\'ed\'erale Polytechnique de Lausanne, Switzerland. Email: \texttt{lenka.zdeborova@epfl.ch}} }

	\date{}
	
	\maketitle
	
	\begin{abstract}%
		We consider the task of estimating a low-rank matrix from non-linear and noisy observations. We prove a strong universality result showing that Bayes-optimal performances are characterized by an equivalent Gaussian model with an effective prior, whose parameters are entirely determined by an expansion of the non-linear function. 
		In particular, we show that to reconstruct the signal accurately, one requires a signal-to-noise ratio growing as \(N^{\frac 12 (1-1/k_F)}\), where \(k_F\) is the first non-zero Fisher information coefficient of the function. 
		We provide asymptotic characterization for the minimal achievable mean squared error (MMSE) and an approximate message-passing algorithm that reaches the MMSE under conditions analogous to the linear version of the problem. 
		We also provide asymptotic errors achieved by methods such as principal component analysis combined with Bayesian denoising, and compare them with Bayes-optimal MMSE. 
		%
	\end{abstract}

	\section{Introduction and Related Work}
	
	Learning a signal from noisy, high-dimensional observations is a ubiquitous problem across statistics, probability, and machine learning. In particular, the task learning a vector $\bf x$ from the noisy, possibly non-linear, observation of the matrix $\bf x \bf x^\top$ has applications ranging from sparse PCA \cite{zou2018selective} to community detection \cite{Blockmodel}, sub-matrix localization \cite{hajek2017information}, or matrix completion \cite{candes2012exact}. Many variants of this problem have been studied in the high-dimensional limit (e.g. see \cite{deshpande2014information,DBLP:journals/corr/DeshpandeAM15,dia2016mutual,lesieur2017constrained,barbier2018rank,lelarge2017fundamental,10.1214/19-AOS1826,mourratrank1,montanari2022equivalence}). The topic has been of particular interest from the random matrix theory point of view, in connection to the Wigner spiked model \cite{johnstone2009consistency}, that has sparked many notable works \cite{donoho1995adapting,peche2014deformed,BBP}.
	
	The existing literature focused  on the following model of non-linear low-rank matrix estimation: 
	Assume ${\bf x} \in {\mathbb R}^N$ to be the unknown target vectors sampled from a prior, we are then given the noisy observation $\rdmmat{Y}$, a $(N \times N)$ matrix generated by applying a nonlinear stochastic function to each component of the rank one matrix $\bf x \bf x^\top$:
	\begin{equation}
		\label{eq:initial_model}
		Y_{ij} \sim\pout \left( \cdot \left|\frac{1}{\sqrt N} x_ix_j\right)\right.
	\end{equation}
	It was established in \cite{krzakala2016mutual} that as long as the probabilistic channel $p_{\rm out}$ has a non-zero Fisher information, that is $\Delta^{-1} = \mathbb E[\partial_w p_{\rm out}(Y|w)|_{w=0}] > 0$, then Gaussian universality  implies that the non-linear problem is equivalent to the spiked Wigner model: 
	\begin{equation}
		\rdmmat{Y}^{(G)} =  \frac{\vect{x} {\vect{x}^{\top}}}{\sqrt{N}} + \sqrt{\Delta} \rdmmat{G},
	\end{equation}
	where $\rdmmat{G}$ is a Gaussian symmetric matrix, component iid each with zero mean and unit variance.
	Indeed, the mutual information of the two problems \cite{DBLP:journals/corr/DeshpandeAM15,krzakala2016mutual}, as well as their minimal mean square errors \cite{guionnet2023estimating} are asymptotically equal. This appealing {\it universality} enables a rigorous mapping to problems such as community detection or sub-graph localization and has been one of the motivations behind the many studies of low-rank matrix factorization.

	
	In this work, we consider the non-linear low-rank matrix estimation problem for cases for which the probabilistic channel $p_{\rm out}$ has {\it zero Fisher information}, that is when $\mathbb E[\partial_w p_{\rm out}(Y|w)|_{w=0}]=0$. This happens, for instance, for even function  (in $w$) $p_{\rm out}$ \cite{guionnet2023spectral}. One then needs to consider the signal to be stronger so that there is {\it enough} information in the matrix $\rdmmat{Y}$ to reconstruct the target. Eq.~\eqref{eq:initial_model} must be modified in the form
	\begin{equation}
		\label{eq:notation_intro}
		Y_{ij} \sim  p_{\rm out}\left(\cdot \given \frac{\gamma(N)}{\sqrt{N}}
		x_i x_j \right)\, ,
	\end{equation}
	with a suitable additional $N$-dependent function $\gamma(N)$. 
	As we shall see, to solve the problem one needs to look at a signal-to-noise ratio growing as \(\gamma(N) \propto N^{\frac 12 (1-1/k_F)}\), where \(k_F\) is given as the first non-zero Fisher information coefficient of the function when expending in powers, that is the first non-zero order $k$ where 
	$\mathbb E[\partial^k_w p_{\rm out}(Y|w)|_{w=0}] \neq 0$.

	Problems of the form (\ref{eq:notation_intro}) have motivations beyond low-rank factorization. In theoretical machine learning, it can be interpreted as a kernel applied to spiked models. Such a situation appears in the studies of gradient descent in neural networks. In this case, the low-rank perturbation originates from the early steps of training with gradient descent \cite{ba2022high,damian2022neural,dandi2023learning}, and the Fisher exponent plays the role of the information exponents of \cite{ben2022high}.  A special case of the problem we study was considered by \cite{9517881}, who studied the information-theoretic aspects of recovering a low-rank signal $x$ from unsigned observations of the matrix, corresponding to $f(x) = |x|$. Recently, \cite{guionnet2023estimating} and \cite{feldman2023spectral} studied spectral methods for such problems. The analysis of the information-theoretically optimal and efficient algorithmic performance for the setting (\ref{eq:notation_intro}), 
	was, however, lacking. This is what we establish in the present paper.

	\paragraph{Main Contributions ---}
	We provide a unified approach to study non-linear spiked matrices. The key contributions stem from universality (Theorem~\ref{th:informal_rank1decompositionFisher}), which states that at a proper scaling for the signal, there exists an explicit entrywise transformation of the observed data, which is completely determined by the structure of the non-linear channel, into a Fisher matrix $\rdmmat{S}_{\kf}=  \rdmmat{Y}_{\kf}/\Delta_{k_F}$ defined in \eqref{eq:fisherscores}. This Fisher  matrix then exhibits a spiked structure (up to some lower-order terms)
	\begin{align}
		\rdmmat{Y}_{k_F}  =  \frac{\vect{x}^{k_F}\left( \vect{x}^{k_F}  \right)^{\top}}{\sqrt{N}}   +\sqrt{ \Delta_{k_F}  } \rdmmat{G}  \, , 
	\end{align}
	where the effective signal-to-noise ratio is given by the critical Fisher coefficient \eqref{eq:fisherscores} and the effective signal is the $\kf$ Hadamard power of the original signal. From this transformation and subsequent decomposition, we have effectively reduced the non-linear spiked matrices, to a {\it classical} spiked Wigner matrix model. The latter has been thoroughly analyzed \cite{deshpande2014information,dia2016mutual,barbier2018rank,10.1214/19-AOS1826,lelarge2017fundamental,10.1214/19-AOS1826,guionnet2023estimating}. 
	
	This underlying spiked structure of the Fisher matrix implies that the Fisher matrix exhibits analogous information theoretical thresholds as the classical spiked Wigner models, allowing us to recover algorithmic and statistical guarantees for this large class of inference problems. Another key consequence is that known spectral algorithms for the spiked Wigner models such as PCA or linearized AMP can be applied to the Fisher matrix to recover the signal. 
	
	\vspace{3mm}
	\noindent From the {\bf information theoretic perspective} we provide the following results: 
	\begin{enumerate}[noitemsep,leftmargin=1em,wide=1pt]
		\item We prove a strong universality result showing that the mutual information / free entropy of non-linear low-rank matrix estimation corresponds to the one of a Gaussian spiked model with a noise level given by the $k$th order Fisher information coefficients. (Theorem~\ref{theo:Gauss_Approximation_Free_Energy}, Theorem~\ref{cor:limiting_value_Free_Energy}), so that the effective signal-to-noise ratio scale as  $N^{(1-1/k_F)/2}$.  This generalizes the previous (limited) universality at $k_F=1$, and provides a rigorous mapping to the original spiked Wigner model 		(\ref{eq:notation_intro}).
		\item We further extend the universality for large deviation rates to conclude that the corresponding Minimal Mean Square Error (MMSE) (for the signal $x^\kf$ where $\kf$ is the critical Fisher coefficient) is also universal, and thus can be computed explicitly in terms of the MMSE for the Gaussian equivalent model. (Theorem~\ref{theo:Gauss_Universality}). 
		\item By using the associated spiked matrix problem characterization, we also provide a fixed point characterization / first-order condition for the optimal overlaps between the $x^\kf$ and a sample from the posterior. This give a fixed point condition of the limiting overlaps, MMSE, of the weak recovery threshold and of the conditions for its existence (see Corollary~\ref{cor:fixedpoint} and Corollary~\ref{cor:recovery}). 
	\end{enumerate}
	We also look to the problem from the {\bf algorithmic perspective}, and provide the following results:
	\begin{enumerate}[noitemsep,leftmargin=1em,wide=1pt]      
		\item We establish that the standard approximate message-passing algorithm (AMP) for Wigner spiked models applied to the Fisher matrix achieves the MMSE under conditions on the prior distribution of the target vector known for the linear version (see Theorem~\ref{theo:AMP_performance}).  
		\item 
		Finally, we also study the performance of a simpler, off-the-shelve method and, in particular, spectral algorithm thanks to a {\it spectral universality} theorem  (Theorem~\ref{th:informal_rank1decompositionFisher}). In particular, we show the Fisher matrix is an optimal method for PCA (Corollary~\ref{theo:informal_PCA_optimality}), and displays a transition for the MSE which we characterize (Prop.~\ref{cor:FPCA_MSE_performance}). We further demonstrated that optimal denoising of the top eigenvector (Corollary~\ref{cor:denoised_FPCA_performance}) yields
		near AMP performances.
	\end{enumerate}
	\begin{figure*}[!t]
		\vskip 0.2in
		\begin{center}
			\includegraphics[width=0.49\linewidth]{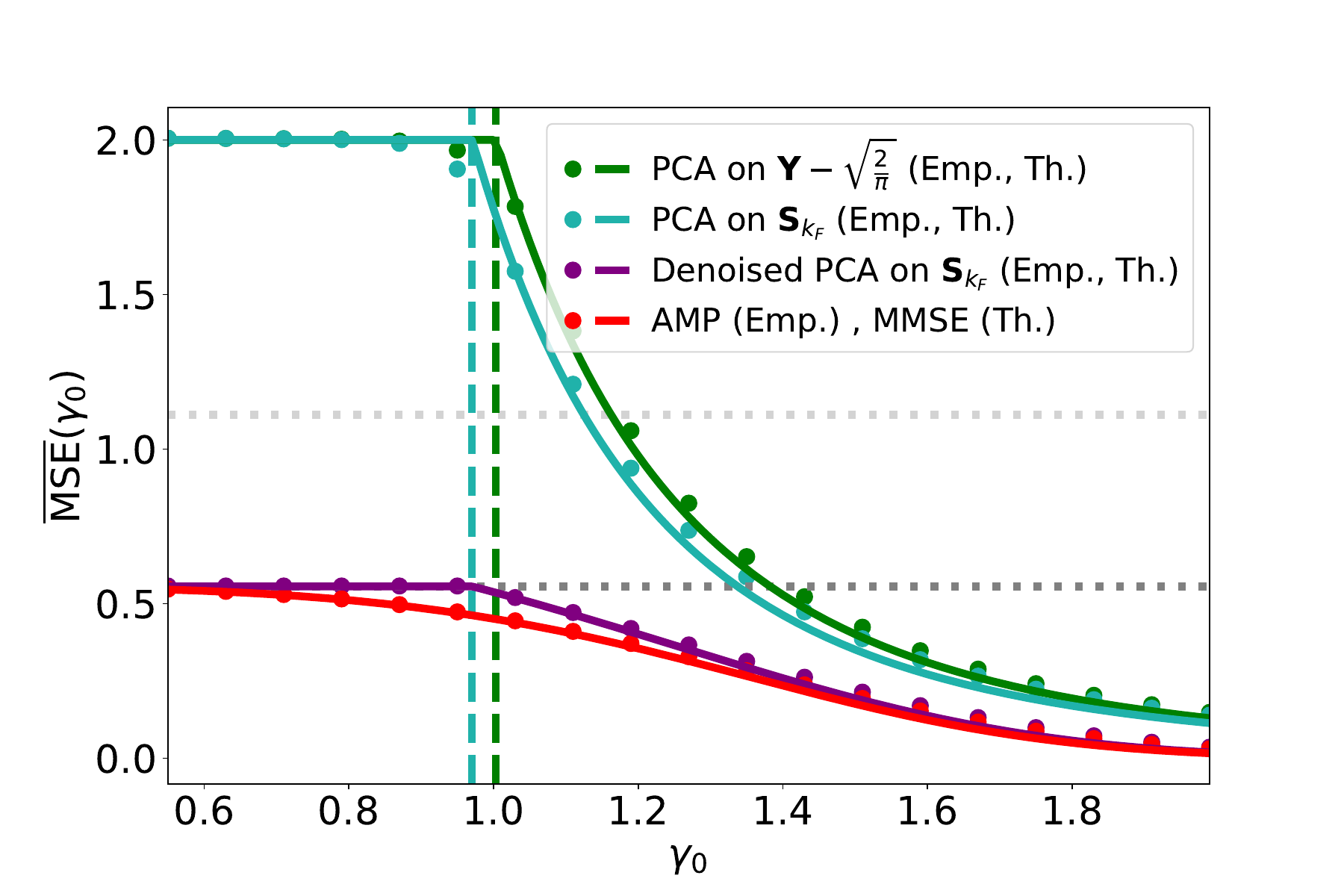}
			\includegraphics[width=0.49\linewidth]{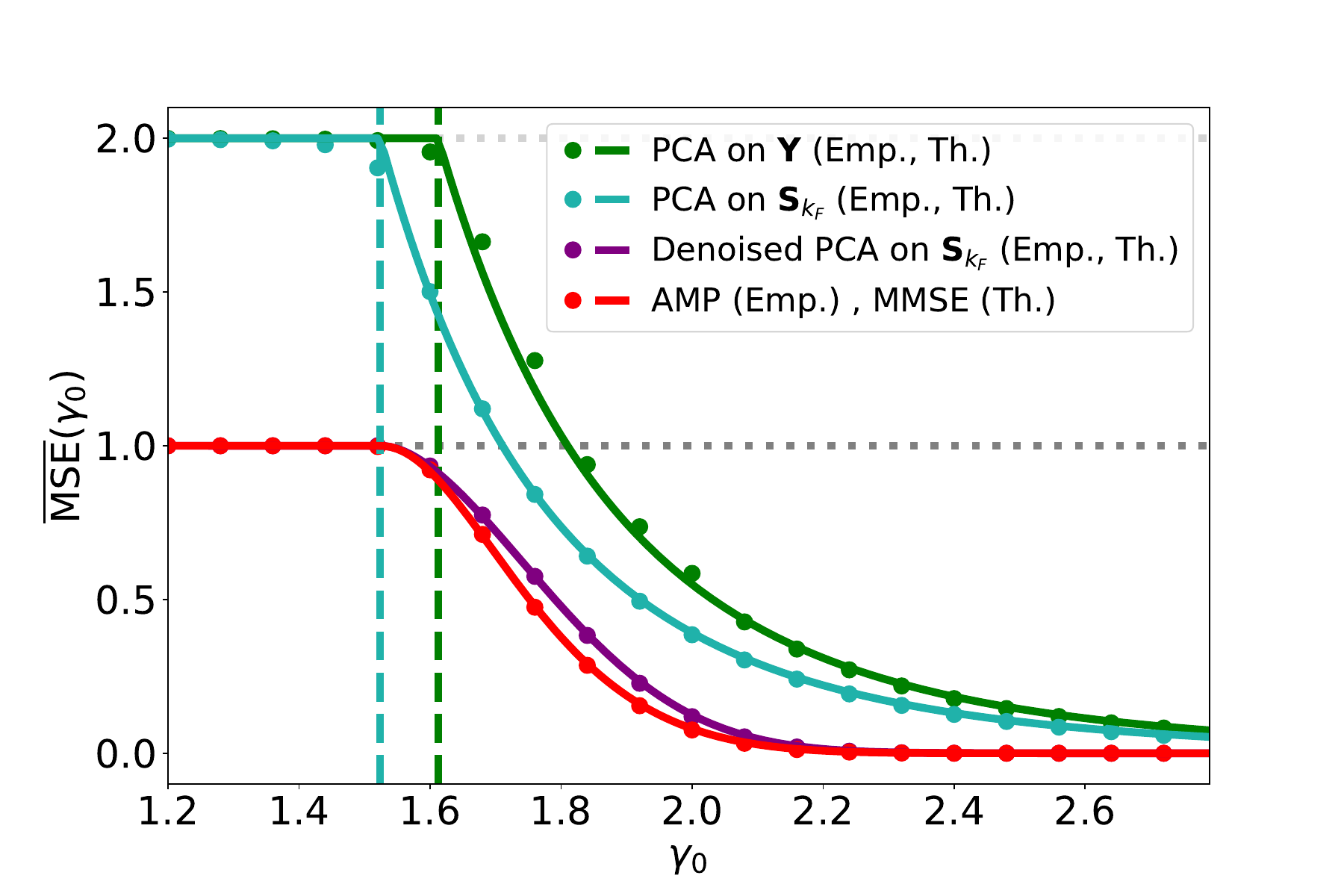}
			\caption{\textbf{Matrix mean squared error for different estimators, for the model of Eq.~\eqref{eq:example_kf2} (left) and the model of Eq.~\eqref{eq:example_kf3} (right), as a function of the signal-to-noise parameter $\gamma_0$}. 
				Both figures illustrate the optimal performance
				achieved by the Approximate Message Passing (red). This is compared to the spectral method on the data matrix $\rdmmat{Y}$ (green); the spectral method on the Fisher matrix $\rdmmat{S}_\kf$ (cyan); and Gaussian denoising of the top eigenvector of $\rdmmat{S}_\kf$ (purple).  
				Note that the last method is very close to optimality. The empirical points (illustrated by dots in both figures) are obtained by doing an average of  $30$ samples with $N=4000$. The phase transitions of the spectral methods are marked with vertical lines. For $\kf=2$ (left panel), there is no phase transition for the MMSE and AMP, since $m_{\kf=2} \neq 0$.  
			}
			\label{fig:MSE}
		\end{center}
		\vskip -0.2in
	\end{figure*}
	These contributions are illustrated in Fig. \ref{fig:MSE}. On {\it the left panel}, we show what happens with the problem where	$Y_{ij} 
	=
	| G_{ij} +\gamma(N)/\sqrt{N} x_i x_j|$ with  $\gamma(N) = \gamma_0 N^{1/4}$, for which $k_F=2$, see Eq.~\eqref{eq:example_kf2}. The MMSE is shown, with the results of the AMP algorithm (red) that reaches it efficiently for all values of $\gamma_0$. The naive spectral estimators display a sharp (green)  suboptimal threshold compared to those based on the Fisher matrix (blue and purple). Note how the denoised PCA of the Fisher matrix (purple) very closely approaches the MMSE at a low computational cost. 
	
	On {\it the right panel}, we present another example where $Y_{ij} =
	Z_{ij} + f_0 (\gamma(N)/\sqrt{N} x_i x_j )$  with $ \gamma(N) = \gamma_0 N^{1/3} $ and $f_0(w) := w - \tanh(w) + \frac{2}{15} w^5 $ and $Z_{ij}$ follows a Student distribution, see Eq.~\eqref{eq:example_kf3}. In this case, $k_F=3$, and the MMSE now displays a sharp phase transition at $\gamma_0=1.52$. As before, the PCA on the Fisher matrix is found to be a better estimator than the naive PCA on the observed matrix, while the denoised PCA of the Fisher matrix is nearly achieving the MMSE. 
	
	\paragraph{Further Related works ---} The low-rank matrix factorization has been connected to many other models ranging from planted cliques \cite{deshpande2015finding,barak2019nearly}, synchronization \cite{javanmard2016phase,perry2018message}, to spin glass models \cite{panchenko2012sherrington,zdeborova2016statistical}, and has been used as a theoretical test bed for many techniques from sun of squares \cite{hopkins2016fast} to spectral methods \cite{perry2018optimality}. The extension to tensor is also very popular \cite{hopkins2015tensor,lesieur2017statistical,arous2019landscape}, and we expect that a similar generalization as the one we present should appear in this case too.
	
	The connection between AMP and the algorithmic hardness of the factorization task has been discussed in a series of works, e.g. \cite{lesieur2015mmse,zdeborova2016statistical,celentano2020estimation,montanari2022equivalence}). It is tightly connected to the behavior of the large deviation free energy, also called the Parisi-Franz potential in statistical physics \cite{franz1995recipes}, see, e.g. \cite{guionnet2023estimating,bandeira2022franz}.
	
	Spectral methods such as PCA are very popular tools for solving such problems. The use of the Fisher matrix, in particular, was discussed in \cite{lesieur2015mmse,lesieur2017constrained} and \cite{perry2018optimality}.  Recently, \cite{guionnet2023estimating} and \cite{feldman2023spectral} studied matrix with the same scaling as in our problem.
	
	\section{Model and Assumptions}
	\label{sec:Model_and_Assumptions}
	\subsection{Notations}
	We use the notations $f \equiv f(a)$ whenever we omit the dependency in the variable $a$ for the function $f$. Equality in law and convergence in law are denoted respectively by $X \ed Y$ and $X \toed Y$. In informal theorems, we use the notation $ A \approx B$ to denote an "approximate equality" between $A,B$ and we will always provide the reference of the precise meaning of this statement. 
	
	We denote by $\mathcal{P}_c(\mathbb{R})$ the set of probability measures with compact support. For $k \in \mathbb{N}_*$ and $P \in \mathcal{P}_c(\mathbb{R})$, we denote by $P_k \in \mathcal{P}_c(\mathbb{R})$ the pushforward of $P$ by the map $x \mapsto x^k$ (thus if $x \sim P$ then $x^k \sim P_k$). We denote by $\pi$ be the \emph{prior distribution} and by $\pout(.|w)$ the \emph{density of a conditional distribution}. We will frequently use the notation $Y_0 \sim \pout(.|0)$ whenever we condition on having no signal ($w=0$), in contrast with   $Y \sim \pout(.|w)$ whenever one has a priori $w \neq 0$.   For $k \in \mathbb{N}_*$ and $P \in \mathcal{P}_c(\mathbb{R})$, we denote by $m_k(P_0):= \mathbb{E}_{x \sim P_0}  \{ x^k \}$ the $k$-th moment of $P_0$. We use $m_k \equiv m_k(\pi)$ whenever the distribution $P_0 = \pi $ is clear from the context. For $\vect{x}$ the \emph{signal vector} and $\beta \in (0,1/2)$, we will frequently use the notations $\mat{W} :=  N^{\beta-1/2} \vect{x} \vect{x}^\top$ and  $ \mat{W}_\kf :=  N^{-1/2} \vect{x}^\kf (\vect{x}^\kf)^\top$ where $\kf \in \mathbb{N}_*$ is defined in the next section. 
	
	Throughout the paper, we will frequently differentiate (random) vectors depending on the noisy data $\rdmmat{Y}$ by a bold italic font (e.g. $\rdmvect{v}$) to vectors that are independent of  $\rdmmat{Y}$ conditioned on the signal $\vect{x}$ denoted with a bold regular font (e.g. $\vect{v}$). We take the same conventions for matrices (e.g. $\rdmmat{A}$  and $\mat{A}$ for respectively a random matrix that is not independent of $\rdmmat{Y}$ and one that is independent of $\rdmmat{Y}$ conditioned on $\vect{x}$). For $k \in \mathbb{N}$ and $\vect{v} \in \mathbb{R}^N$, we denote by $\vect{v}^k := (v_1^k,\dots,v_N^k)$, the entrywise power of $\vect{v}$. The scalar product is denoted by $\langle \vect{u}, \vect{v} \rangle := \sum_{i} u_i v_i $ and the L2-norm by $\| \vect{v} \| := \sqrt{\langle \vect{v} , \vect{v} \rangle}$. 
	
	For $f:\mathbb{R} \to \mathbb{R}$ we denote with a bold symbol  $\vect{f}:\mathbb{R}^N \to \mathbb{R}^N$ the \emph{vectorized} version of $f$ with respect to its argument, that is for $\vect{a} =(a_1,\dots,a_N) \in \mathbb{R}^N$, $\vect{f}(\vect{a}) := (f(a_1),\dots, f(a_N)) \in \mathbb{R}^N$. For $\vect{h}:\mathbb{R}^N \to \mathbb{R}^N$ and differentiable entrywise, we denote by $(\nabla_{\vect{a}} \cdot \vect{h} )(\vect{a}) := \sum_{i=1}^N h'_i(a_i)$ the divergence operator applied to $ \vect{h}$. 
	
	\subsection{Structure of the Model}
	\label{sec:Model}
	Under Hypothesis postponed to Sec.~\ref{sec:Assumptions}, we assume one observes the rank-one matrix of the signal through the noisy channel given by $\pout$, that is, one observes the symmetric matrix $	\rdmmat{Y} = (Y_{ij})_{1 \leq i,j \leq N}$ such that for $1 \leq i \leq j \leq N$, the  entry $Y_{ij}$ is distributed according to
	\begin{align}
		\label{eq:notation_rank1_mat}
		Y_{ij} \sim p_{\text{out}}(\cdot | W_{ij})
		\quad \mbox{where} \quad  	\rdmmat{W} 
		&=
		(W_{ij})_{1 \leq i,j \leq N} = \frac{N^{\beta}}{\sqrt{N}}\, 
		\vect{x} {\vect{x}^{\top}} \quad \mbox{with} \quad   \vect{x} \sim \prior^{\otimes N}(\cdot) \, ,
	\end{align}
	from some $\beta \in [0,1/2)$. 
	
	\subsection{Definitions}
	For $P_0 \in \mathcal{P}_c(\mathbb{R})$,  we introduce 
	the \emph{Gaussian denoising function (with prior $P_0$)} as  $\eta_{P_0}: (a,b) \in (\mathbb{R}_+,\mathbb{R}) \mapsto \eta_{P_0}(a,b) \in \mathbb{R}$ by:
	\begin{align}\label{eq:Gauss_denoise}
		\eta_{P_0}(a,b) := \frac{ \E_{x \sim P_0}  \left[ x \, \mathrm{e}^{ -a x^2/2 + bx}  \right] }{  \E_{x \sim P_0}  \left[\mathrm{e}^{ -a x^2/2 + bx}  \right]  }  \, .
	\end{align}
	We will also frequently use its vectorized version for its second argument $\vect{\eta}_{P_0}(a,\vect{b}) \in \mathbb{R}^N$ for $\vect{b} \in \mathbb{R}^N$. For $P_0 \in \mathcal{P}_c(\mathbb{R})$,  we define the \emph{State-Evolution map} $\mathsf{SE}_{P_0} :  (\mathbb{R}_+ , \mathbb{R}_+) \mapsto \mathbb{R}_+$  by 
	\begin{align}
		\label{eq:def_State_Evolution_map}
		\mathsf{SE}_{P_0}(q , \Delta) := \E_{\Tilde{x} \sim P_0,G \sim \mathsf{N}(0,1)} \, \left[ \eta \left( \frac{q}{\Delta}, \frac{q}{\Delta} \tilde{x} + \sqrt{\frac{q}{\Delta}} G \right) \tilde{x} \right]\,.
	\end{align}
	This function is known to be strictly increasing for its first argument (see for instance 
	Thm.~3.10 of \cite{feng2022unifying}).Recalling $\pout$ is the density of the conditional distribution of the output given the signal  in Eq.~\eqref{eq:notation_rank1_mat} we define the \emph{log-likelihood function} as 
	\begin{align}
		\label{eq:def_log_conditional_distribution}
		g(y,w) 
		:=
		\log p_{\text{out}}( Y_{ij} =y | W_{ij} = w) \, .
	\end{align}
	The function $g(y,w)$ plays a fundamental role in the study of the asymptotic inference of the rank-one inference. We postpone the assumptions on this function to Sec.~\ref{sec:Assumptions} and assume for now that $g$  and its derivatives are sufficiently smooth and bounded such that the following quantities are well-defined, at least for $k$ up to some value $\kf$ defined below.  We use the  shorthand notation $g^{(k)}(y,0) \equiv \partial_w^k g(y,w)|_{w=0}$. 
	\vskip 0.2cm
	\noindent The \emph{$k$-th order Fisher matrix} $\rdmmat{S}_k$ and the $k$-th \emph{Fisher coefficient} $\Delta_k$ are defined as:
	\begin{align}
		\label{eq:fisherscores}
		\rdmmat{S}_k 
		&:=
		\left( S_{ij,k} = \frac{ g^{(k)}(Y_{ij},0)}{k!} \right)_{1 \leq i,j \leq N} \quad \mbox{and} \quad  \frac{1}{\Delta_k} 
		:=
		\E_{Y_{ij}|0} \left[ \left( \frac{ g^{(k)}(Y_{ij},0)}{k!}\right)^2 \right] \, ,
	\end{align}
	and \emph{the critical Fisher index} $\kf$ as 
	\begin{align}
		\label{eq:def_kf}
		k_F
		:=
		\inf \left\{  k \in \mathbb{N}  \quad \mbox{such that} \quad  \frac{1}{\Delta_k} \neq 0 \right\}
		\, .
	\end{align}
	Note that from the definition Eq.~\eqref{eq:fisherscores} of $\Delta_k$,  for $k < \kf$, the condition $(\Delta_k)^{-1} = 0$ is equivalent to $ g^{(k)}(Y_0,0) \eas 0$ for $Y_0 \sim \pout(\cdot | 0)$ . 
	Eventually, we  introduce the \emph{relevant scaling} as: 
	\begin{align}
		\label{eq:criticalscaling}
		\beta_{\mathrm{cr}} = \frac{1}{2} \left( 1 - \frac{1}{\kf}\right) \, .
	\end{align}

	\subsection{Assumptions}
	\label{sec:Assumptions}
	We assume the prior distribution $\pi$ and the conditional distribution $\pout$ to satisfy the following two assumptions.  
	\begin{hyp}[The Prior Signal is Bounded]
		\label{hyp:prior}
		$\pi \in \mathcal{P}_c(\mathbb{R})$ and is independent of $N$. 
	\end{hyp}
	\begin{hyp}[Smooth conditions for the log-likelihood distribution]
		\label{hyp:bayes}
		There exists $K \in \N$ such that $g(y,.) \in C^{2K + 1}$, the quantity $\E_{Y_{ij} | 0} | g^{(k)}(Y_{ij},0) |$ are all bounded $k \in \{K + 1, \dots, 2K \}$,  $\kf \leq K$  and the quantities $g^{(2 \kf)}(y,0)$, $g^{(2 \kf+1)}(y,0)$ are also uniformly bounded.  
	\end{hyp}
	These assumptions provide a rich framework and are convenient for our analysis. We expect however that one can be able to refine the sets of assumptions by, for example, only assuming $\pi_\kf$ to be sub-Gaussian or with a  stretch-exponential decay as in the spectral analysis of \cite{guionnet2023spectral} and we leave this problem for further investigation. 
	
	\subsection{Examples}
	\label{sec:examples}
	\textbf{The (Artificial) $\kf$-Gaussian Additive Channel --- } let us fix a $\kf \in \mathbb{N}$ and $\Delta >0$ and consider the model given by:
	\begin{align}
		\label{eq:def_Gauss_channel}
		Y_{ij} 
		& \ed \sqrt{\Delta} G_{ij} + \frac{1}{\sqrt{N}} x_i^{\kf} x_j^{\kf}  \qquad \mbox{where }  G_{ij} \iid \mathsf{N}(0,1) \quad \mbox{and }   \vect{x} \sim \prior^{\otimes N}(\cdot) \, .
	\end{align}
	By construction, the dependency in the exponent $\kf$ - which corresponds to the critical Fisher index here, hence the notation - is artificial for this model since if one re-defines $\vect{\tilde{x}}:= \vect{x}^{\kf}  \sim \prior_\kf^{\otimes N}$, one ends up with the standard Gaussian additive case for this new signal $\vect{\tilde{x}}$. The Fisher matrix for this model is   the spiked matrix $\rdmmat{S}_\kf =  \rdmmat{Y}/\Delta$, and we have $\Delta_\kf = \Delta$. 
	\jump
	While the dependence in $\kf$ in the previous example seems a tad artificial, our main results (see next section) show that both from the information-theoretic point-of-view and the algorithmic point-of-view, there is a universal mapping between  \emph{non-trivial cases} with critical index $\kf$ and Fisher coefficient $\Delta_\kf$ and this $\kf$-Gaussian additive channel with $\Delta = \Delta_\kf$.
	\vskip 0.2cm
	\noindent \textbf{Example with $\kf=2$ --- } As an illustrative example of this class, consider:
	\begin{align}
		\label{eq:example_kf2}
		Y_{ij} 
		& \ed 
		\bigg | G_{ij} + \frac{\gamma(N)}{\sqrt{N}} x_i x_j  \bigg|  \quad \mbox{with } \gamma(N) = \gamma_0 N^{1/4} \quad \mbox{and } \,  G_{ij} \iid \mathsf{N}(0,1) \, .
	\end{align}
	Since $\pout(y,w) = \pout(y,-w)$, we have $1/\Delta_{k=1} =0$ and one can check that $\kf =2$ for this model. The expression of $\gamma(N)$ in Eq.~\eqref{eq:example_kf2} has been chosen such that one is precisely in the relevant scaling $\beta = \beta_\mathrm{cr}$. The Fisher matrix for this model is given by $S_{\kf,ij} = \frac{\gamma_0^2}{2} (Y_{ij}-1)^2$ and $\frac{1}{\Delta_\kf} = \frac{\gamma_0^4}{2}$. For numerical simulations, as illustrated in Fig.~\ref{fig:MSE} and Fig.~\ref{fig:Eigengap}, we pick $\pi = (\delta_{a} +  \delta_b)/2$ with $a= \sqrt{2 -\sqrt{2}}/2$, $b= -\sqrt{2 +\sqrt{2}}/2$. Note that  $|a| \neq |b|$ to avoid the degenerate case where $\pi_\kf$ is a Dirac mass at one point.
	\vskip 0.2cm
	\noindent   \textbf{Example with $\kf=3$ --- }  We consider 
	\begin{align}
		\label{eq:example_kf3}
		Y_{ij} 
		& \ed 
		Z_{ij} + f_0 \bigg(\frac{\gamma(N)}{\sqrt{N}} x_i x_j \bigg)   \quad \mbox{with } \gamma(N) = \gamma_0 N^{1/3} \, , 
	\end{align}
	where $f_0(w) := w - \tanh(w) + \frac{2}{15} w^5  \underset{w \to 0}{\sim} w^3$ such that $\kf =3$ and $Z_{ij} \iid \mathsf{Student}(\nu)$. The Fisher matrix and coefficient for this model are $S_{\kf,ij} =  \gamma_0^3  \frac{1+\nu}{3} \frac{Y_{ij}}{Y_{ij}^2+\nu}$ and $\frac{1}{\Delta_\kf}=   \gamma_0^6 \frac{ (1+\nu)}{9 (3+\nu)}$. For numerical simulations (Fig.~\ref{fig:MSE} and Fig.~\ref{fig:Eigengap}), we have chosen $\nu =4.1$ and $\pi = (\delta_{-1} + \delta_1)/2$.

	\begin{figure*}[t]
		\vskip 0.2in
		\begin{center}
			\includegraphics[width=0.5\linewidth]{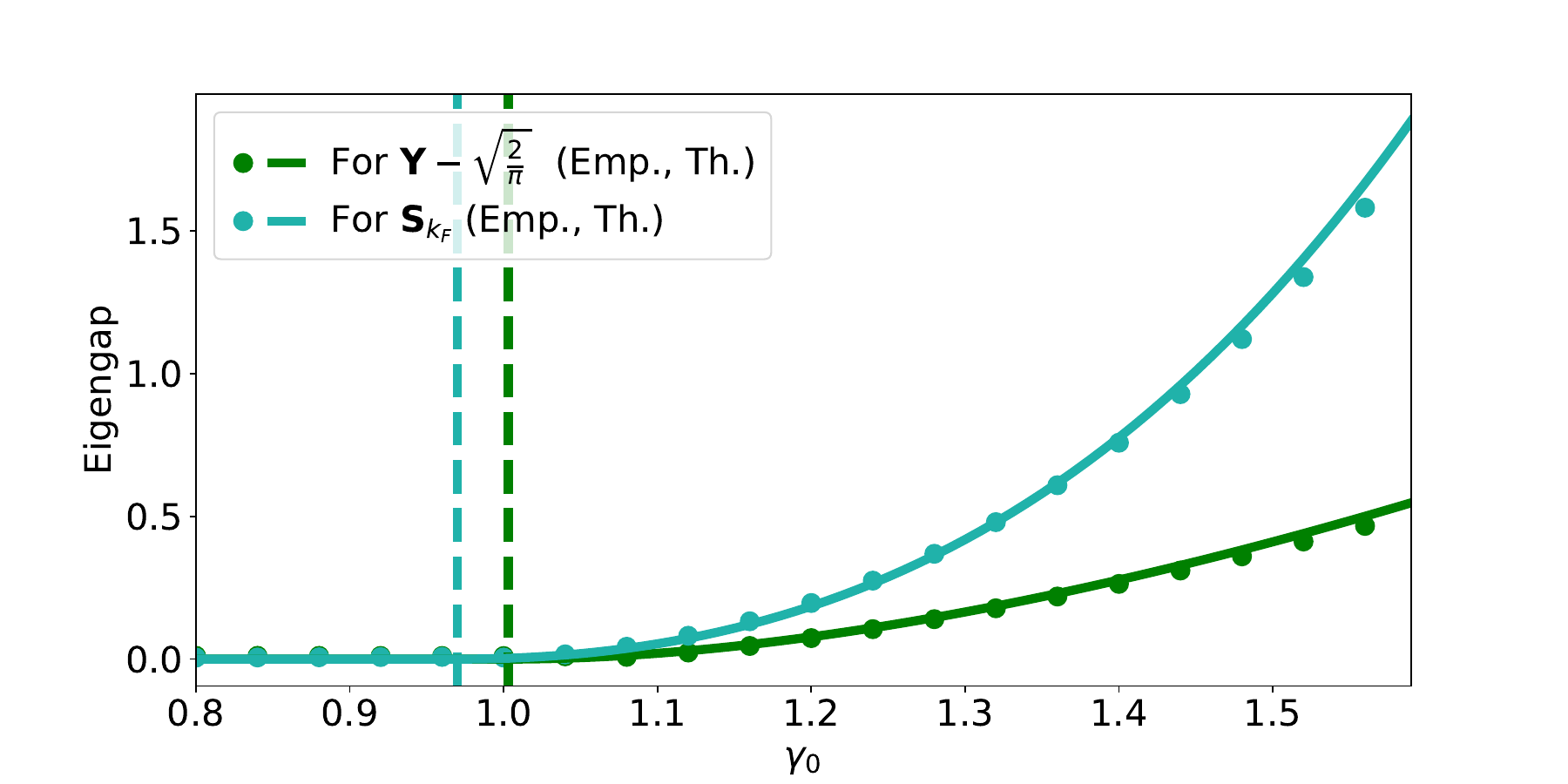}
			\includegraphics[width=0.49\linewidth]{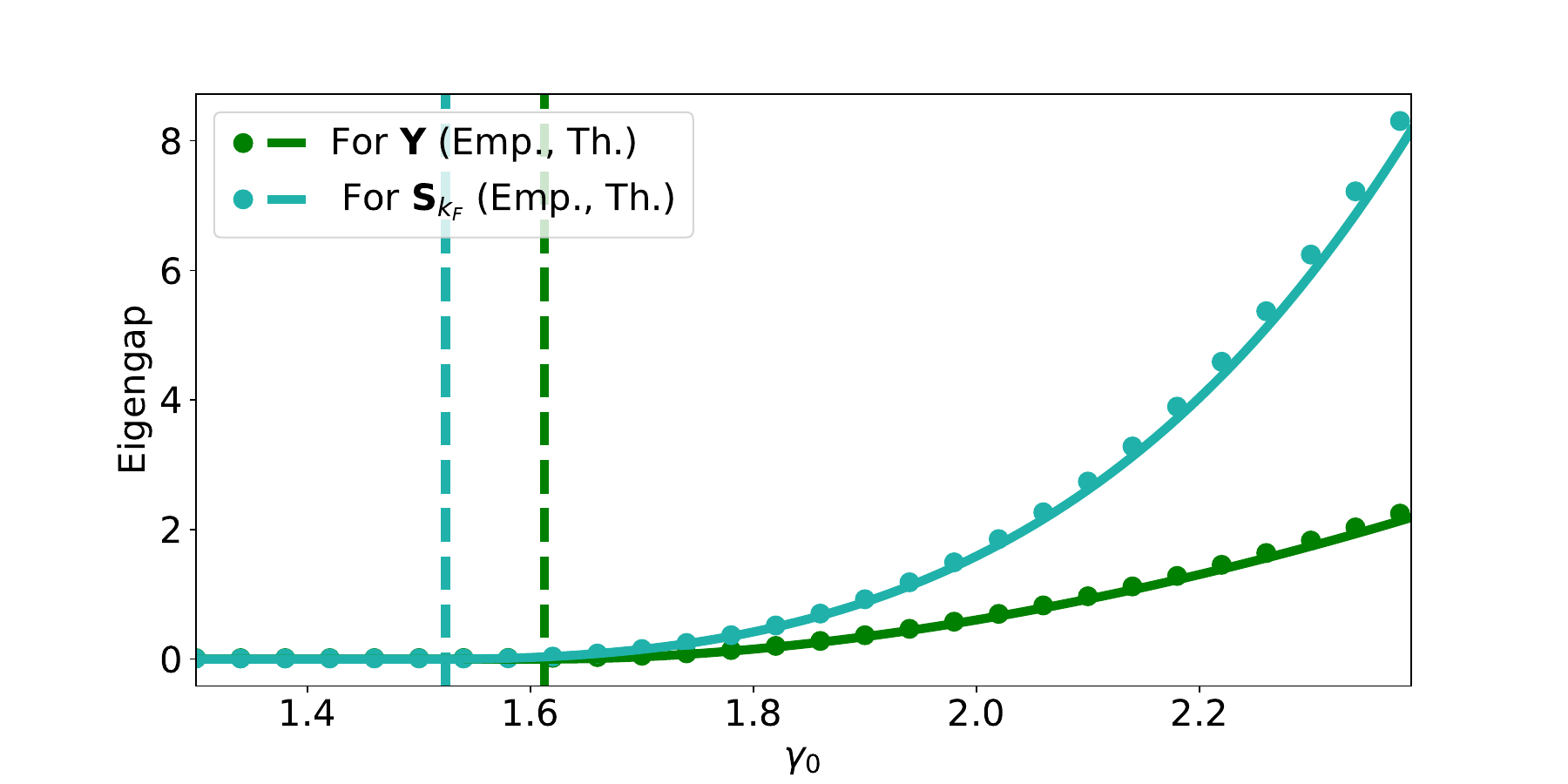}
			\caption{\textbf{Difference between the largest and second largest eigenvalue (Eigengap) of $\rdmmat{Y}$ and $\rdmmat{S}_\kf$, as a function of the parameter $\gamma_0$ for the model of Eq.~\eqref{eq:example_kf2} (left) and the model of Eq.~\eqref{eq:example_kf3} (right).} The threshold for the appearance of an outlier, that is the minimal value of $\gamma_0$ such that one has a positive value of the eigengap, is lower for the Fisher matrix than the one of the output matrix $\rdmmat{Y}$ in both models.  Each empirical point (dot) corresponds to one sample with $N=2000$.  In the left figure, the output matrix is shifted by a constant to remove the non-informative Perron-Frobenius mode.  }
			\label{fig:Eigengap}
		\end{center}
		\vskip -0.2in
	\end{figure*}

	\section{Main results}
	\label{sec:Main_results}

	\subsection{Decomposition of the Fisher Matrix and Its Top Eigenvector}
	\label{sec:Spike_decomposition_Fisher}
	
	Our first series of results concerns the spectral properties of the Fisher matrix $\rdmmat{S}_{k_F}$. This matrix appears as a first-order approximation of the log-likelihood of the original model defined in  Eq.~\eqref{eq:notation_rank1_mat}.
	
	\begin{theo}[Informal - Universal Spiked Decomposition of the Fisher matrix]
		\label{th:informal_rank1decompositionFisher}
		Under Hyp.~\ref{hyp:prior} and \ref{hyp:bayes} with the relevant scaling $\beta = \beta_{\mathrm{cr}}$ in Eq.~\eqref{eq:notation_rank1_mat}, we have: 
		\begin{align}
			\frac{\rdmmat{S}_{k_F}}{\sqrt{N}}  \approx   \frac{\rdmmat{Z}}{\sqrt{N}} +\frac{1}{\Delta_{k_F}}  \frac{\vect{x}^{k_F}}{\sqrt{N}} \left( \frac{\vect{x}^{k_F}}{\sqrt{N}}  \right)^{\top}  \, , 
		\end{align}
		where $\rdmmat{Z} = (Z_{ij})_{ 1 \leq i,j \leq N}$ is a Wigner matrix with elements of variance $1/\Delta_{k_F}$.
	\end{theo}
	
	\begin{proof}
		Through the higher order Fisher identities (App.~\ref{sec:higherorderfisher}), we can Taylor expand with respect to the spike and prove a version of the rank~1 equivalence in \cite{guionnet2023spectral} for the Fisher matrix. The proof and rigorous statement of this result is given in App.~\ref{sec:Proof_Spiked_Fisher}. 
	\end{proof}
	\noindent  The structure of the top eigenvector of $\frac{\rdmmat{S}_{k_F}}{\sqrt{N}}$ is explicit, and admits its decomposition into a part that is correlated to the signal vector, and an independent Gaussian noisy part. If we introduce the \emph{spectral overlap}
	\begin{align}
		\label{eq:def_limoverlap_PCA_Fisher}
		\mathrm{q}_{0} \equiv \mathrm{q}_{0}(\Delta_{k_F})
		:= 
		\begin{cases}
			0 & \mbox{if} \quad \Delta_{k_F} \geq (m_{2 k_F})^2 ,\\
			\sqrt{ 1 - \frac{\Delta_{k_F} }{(m_{2 k_F})^2 }}  &\mbox{if} \quad \Delta_{k_F} < (m_{2 k_F})^2 .
		\end{cases}
	\end{align}
	then we have the following result; 
	\begin{theo}[Informal - Universal Top Eigenvector Decomposition of the Fisher Matrix]\label{theo:informal_eigvector_decomp_Fisher}
		Under the same setting of Thm.~\ref{th:informal_rank1decompositionFisher}, if $\rdmvect{v}_1$ is the top eigenvector of the Fisher matrix $\rdmmat{S}_{\kf}$, we have:
		\begin{align}
			\sqrt{N} \rdmvect{v}_1 \approx 
			\mathrm{q}_0 \, \bigg( \frac{\vect{x}^{\kf}}{\sqrt{m_{2 \kf}} } \bigg) +  \sqrt{1- \mathrm{q}_0 ^2} \,  \vect{g} \, ,
		\end{align}
		with $ \vect{g}$ a standard Gaussian vector independant of $\vect{x}^{\kf}$, and $q_0$ is defined by Eq.~\eqref{eq:def_limoverlap_PCA_Fisher}.
	\end{theo}
	\begin{proof}
		This follows from a characterization of the outlier eigenvalues in \cite{benaych-georges_eigenvalues_2011} and a CLT for these outliers in \cite{PizzoRenfrew}. The proof and rigorous statement of this result is given in App.~\ref{sec:Proof_eigvect_decomposition}. 
	\end{proof}
	Note that in particular, we have $ |\langle \vect{v}_1 , \vect{x}^\kf/\|  \vect{x}^\kf \| \rangle | \to \mathrm{q}_0$.

	\subsection{Information-Theoretic Limit: Higher-Order Universality of the Free Energy and MMSE.}
	\label{sec:Universality_Free_Energy}
	\paragraph{Higher-Order Universality of the Free Energy --- } The next series of results are related to the fundamental information-theoretic limits of non-linear low-rank matrix estimation problems.
	We now introduce the free energy, which is the expected log-likelihood ratio (under data generated by the model) between the likelihood that data are generated by the present model with the likelihood that they are generated from the null model (where there is no signal at all), 
	\begin{definition}[Free Energy]
		\label{def:Free_Energy}
		For a generic output channel of the form of \eqref{eq:nonlinear_channel}, the \emph{associated free energy (per variable)} is defined as
		\begin{align}\label{eq:freeenergybayes}
			F_N(g,\pi) 
			&:=
			\frac{1}{N} \bigg( \E \log \int e^{ \sum_{i < j} g(Y_{ij},M_{ij}) } \,  \pi^{\otimes N}( \vect{m} ) - \sum_{i < j} \E g(Y_{ij},0) \bigg) \, , 
		\end{align}
		where the integrand is denoted by $\mat{M} =  N^{\beta -1/2}	\vect{m} {\vect{m}^{\top}} $ to not confuse it with the signal matrix $\mat{W}$.
	\end{definition}
	Note that, up to the conditional entropy, this quantity is \emph{the mutual information} $I(\vect{x},\rdmmat{Y})/N$ between the signal and the observation. The conditional entropy is easy to compute compared to the free energy, whose limit is non-trivial.
	\jump For the $\kf$-Gaussian additive channel of Sec.~\ref{sec:examples}, if we denote $F^{(G)}_N(\Delta, \pi_\kf) \equiv F_N (g , \pi_\kf)$,  this free energy can be written as 
	\begin{align}
		F^{(G)}_N(\Delta_ , \pi_\kf) =  \frac{1}{N}  \E \log \int e^{\frac{1}{2 \Delta} \Tr( (\sqrt{\Delta} \, \rdmmat{G} + \mat{W}) \mat{M}_\kf ) -  \frac{1}{4 \Delta} \Tr(  \mat{M}_\kf \mat{M}_\kf )  } \, \dpriork^{\otimes N}( \vect{m}  ) + c_N \, , 
	\end{align}
	with $\mat{M}_\kf = N^{-1/2} \vect{m} \vect{m}^\top$ and $c_N = O(N^{-1})$  (to remove the diagonal term $i=j$ in the trace).  
	The limiting free energy $  \mathsf{F}^{(G)}( \Delta, \pi_\kf) := \lim_{N \to \infty} F^{(G)}_N(\Delta, \pi_\kf)$  is given  (from \cite{dia2016mutual,lelarge2017fundamental}) by
	\begin{align}\label{eq:gaussianFE}
		\mathsf{F}^{(G)}( \Delta, \pi_\kf) = \sup_{q \geq 0} \mathcal{F}(\Delta,q) \, , 
	\end{align} 
	where the \emph{replica symmetric functional $\mathcal{F}$} is equal to
	\begin{align}
		\label{eq:RSformula}
		\mathcal{F}(\Delta,q) \equiv \mathcal{F}(\Delta,q; \pi_\kf)  &:= - \frac{q^2}{4\Delta} +  \E_{\tilde{x},G}\log \bigg[ \int \exp \bigg(  \frac{\sqrt{q}}{\Delta} G m +   \frac{\tilde{x} m}{\Delta}  - \frac{m^{2}}{2 \Delta}  \bigg) \,  \dpriork(m) \bigg] \, ,
	\end{align}
	where $\Tilde{x} \sim \pi_\kf, G \sim \mathsf{N}(0,1)$ . The value $\mathrm{q}_{\star}$ maximizing the variational problem on the right-hand side of \eqref{eq:gaussianFE} is \emph{one} solution of the fixed point equation
	\begin{align}\label{eq:qstar}
		\mathrm{q}_{\star} 
		=
		\mathsf{SE}_{\pi_\kf}(\mathrm{q}_{\star} , \Delta) \, ,
	\end{align}
	where $ \mathsf{SE}_{\pi_\kf}(\mathrm{q}_{\star} , \Delta)$ is given Eq.~\eqref{eq:def_State_Evolution_map}. We prove that in the high dimensional limit, the free energy \eqref{eq:freeenergybayes} is equivalent to the free energy of this $\kf$-Gaussian additive model, namely we have: 
	\begin{theo}[Gaussian Approximation of the Free Energy]
		\label{theo:Gauss_Approximation_Free_Energy} Under Hyp.~\ref{hyp:prior} and \ref{hyp:bayes} and with $\beta=\beta_\mathrm{cr}$, we have:
		\begin{align}
			|F_N(g,\pi) -  F^{(G)}_N(\Delta_{k_F} , \pi_{\kf}) | 
			&\leq 
			O (N^{-1/(2 \kf)}) \, .
		\end{align}
	\end{theo}
	\begin{proof}
		We identify the leading order terms of the posterior through an application of the  Fisher identities (see App.~\ref{sec:higherorderfisher}). Through these identities, we generalize the Gaussian universality in \cite{AJFL_inhomo,guionnet2023estimating} at the level of the free energy, which deals with the case $\kf = 1$. The full proof of this result is given in App.~\ref{sec:Proof_Universality}. 
	\end{proof}
	As a consequence, we have in particular:
	\begin{cor}[Universality of the Limiting Free Energy]
		\label{cor:limiting_value_Free_Energy}
		Under Hyp.~\ref{hyp:prior} and \ref{hyp:bayes} and with $\beta=\beta_\mathrm{cr}$, we have:
		$\lim_{N \to \infty} F_N(g,\pi) =\mathsf{F}^{(G)}( \Delta_{\kf}, \pi_{\kf}) $. 
	\end{cor}
	
	\paragraph{Higher-Order Universality of the Overlap and MMSE --- } Switching the sign of the vector $\vect{x}$ leaves the conditional law of $Y_{ij}$ given by Eq.~\eqref{eq:notation_rank1_mat} invariant as the latter only depends on the product $x_i x_j$. Thus one can, in general, only estimate the vector  $\vect{x}$ up to a sign, and a natural measure of performance for such an estimator is given in this setting by the \emph{(Matrix) Mean Square Error} (MSE). 
	
	\begin{definition}[Matrix Mean Squared Error]
		For a vector $\vect{u}$ and an estimator $\rdmvect{\hat{u}}$ of $\vect{u}$, we define the (normalized) matrix square error between $\vect{u}$ and $\rdmvect{\hat{u}}$  as:
		\begin{align}
			\overline{\mathrm{MSE}}_{N}(\vect{u}, \rdmvect{\hat{u}} ) := \frac{  \E
				\|   \vect{u} \vect{u}^{\top}  -  \rdmvect{\hat{u}} \rdmvect{\hat{u}}^{\top}  \|^2_F}{  \E \| \vect{u} \vect{u}^{\top}  \|^2_F} =  \frac{  \E \langle \vect{u},\vect{u} \rangle^2 +  \E \langle \rdmvect{\hat{u}}, \rdmvect{\hat{u}} \rangle^2  - 2   \E \langle \vect{u}, \rdmvect{\hat{u}}\rangle^2 }{   \E \langle \vect{u},\vect{u} \rangle^2} \, ,
		\end{align}
		where $\| . \|_F^2 := \Tr (.)^2 $ is the Frobenius norm.  
	\end{definition}
	From the Bayes optimal point of view, the best possible estimator of $\vect{x}^{\kf}$ one can construct from the noisy observation of $\rdmmat{Y}$ is given by the posterior mean:
	\begin{align}
		\vect{\hat{x}}^{\kf}_{\mathrm{BO}} :=  \underset{\vect{\hat{u}}}{\mathrm{argmin}} 	\, \overline{\mathrm{MSE}}_{N}(\vect{x}^{\kf}, \rdmvect{\hat{u}} ) = \mathbb{E} \left[ \vect{x}^{\kf} | \rdmmat{Y}  \right] \, , 
	\end{align}
	and our next results characterize the information-theoretic value of the \emph{minimum mean square error} (MMSE). We have that the overlaps are universal because they also encode the Bayesian optimal overlaps and MSE of the original models.  
	\begin{theo}[Universality of the Overlaps]
		\label{theo:Gauss_Universality} Under Hyp.~\ref{hyp:prior}
		and \ref{hyp:bayes} and with $\beta=\beta_\mathrm{cr}$,	if $\mathrm{q}_{\star}$ is again the maximizer of the replica free energy, then the overlaps between the posterior and the signal satisfies
		\begin{align}
			\lim_{N \to \infty}	\Bigg  | \Bigg \langle \frac{\vect{x}^{\kf}}{\| \vect{x}^{\kf}\|}, \frac{\rdmvect{\hat{x}^{\kf}}_{\mathrm{BO}}}{\| \rdmvect{\hat{x}^{\kf}}_{\mathrm{BO}} \| } \Bigg \rangle \Bigg |  = \frac{\mathrm{q}_{\star}}{m_{2\kf}}.
		\end{align}
	\end{theo}
	\begin{proof}
		The proof of this result is relies on a universality principle for the inner products in App.~\ref{sec:Proof_Universality} and a characterization of the limit in App.~\ref{sec:optimalMMSE}. Both arguments stem from the universality of the large deviation of the free energy, as in \cite{guionnet2023estimating}.
		
	\end{proof}
	This convergence of overlaps leads to the following result for the MMSE.
	\begin{cor}[Universality of the MMSE] Under the same hypothesis and notations as in Thm.~\ref{theo:Gauss_Universality}, we have:
		\[
		\lim_{N \to \infty} \overline{\mathrm{MSE}}_{N}(  \vect{x}^{\kf} , \rdmvect{\hat{x}^{\kf}}_{\mathrm{BO}} ) = 1 -\bigg(\frac{\mathrm{q}_{\star}}{m_{2\kf}} \bigg)^2.
		\]	
	\end{cor}

	\subsection{Algorithmic Limit: Higher-Order Universality of Spectral Methods and AMP}
	\paragraph{Estimators based on Spectral Methods of the Fisher Matrix --- } We now described how the explicit decomposition of the Fisher matrix allows one to design a hierarchy of spectral estimators. Our first result concerns the performance of the classical spectral estimator, which is an immediate consequence of Thm.~\ref{theo:informal_eigvector_decomp_Fisher}. 
	\begin{proposition}[Performance of PCA of the Fisher Matrix]
		\label{cor:FPCA_MSE_performance}
		Under the same setting of Thm.~\ref{th:informal_rank1decompositionFisher}, if $\rdmvect{v}_1$ is the top eigenvector of $\rdmmat{S}_\kf$, then we have:
		\begin{align}
			\lim_{N \to \infty}\overline{\mathrm{MSE}}_{N} \bigg(\vect{x}^{\kf},  \sqrt{N m_{2 \kf}} \rdmvect{v}_1 \bigg)
			&= 
			2 \left( 1- \mathrm{q}_0^2 \right)
			=
			\begin{cases}
				2  &\mbox{if} \quad \Delta_{k_F} \geq (m_{2 k_F})^2 \, , 
				\\
				2\frac{\Delta_{k_F}}{m_{2 k_F}^2} &\mbox{if} \quad \Delta_{k_F} <(m_{2 k_F})^2 \, .
			\end{cases}
		\end{align}  
		where $\mathrm{q}_0$ is defined by Eq.~\eqref{eq:def_limoverlap_PCA_Fisher}. 
	\end{proposition}
	In the case of a \emph{linear channel} where $Y_{ij} =  Z_{ij}  + \gamma x_i x_j$ but $Z_{ij}$ is not necessarily Gaussian, under the suitable condition on the distribution of $Z_{ij}$, then the optimal entrywise function one can apply to $Y_{ij}$ to improve the threshold in $\gamma$ for the emergence of an outlier in the spectrum, is the one corresponding to the Fisher matrix \cite{perry2018optimality}, as originally proposed on \cite{lesieur2015mmse}. Our next result shows that this extends to non-linear channels. 
	\begin{theo}[Informal - Optimality of PCA for Non-Linear Channels]\label{theo:informal_PCA_optimality} Assume that the channel $\pout$ is such that  $Y_{ij} \ed h \Big(Z_{ij} + \frac{\gamma(N)}{\sqrt{N}} x_i x_j \Big)$ with $Z \iid p_Z$. Under smooth conditions on $h$ and $p_Z$, the optimal function $f$ one can apply entrywise to $\rdmmat{Y}$ to lower the threshold and the MSE is the one of the Fisher matrix.   
	\end{theo}
	\begin{proof}
		We optimize the characterization of the BBP transition for non-linear Wigner spiked models from \cite{guionnet2023spectral}. The details of the are given in Appendix~\ref{sec:proof_optimal_PCA}. 
	\end{proof}
	The comparison between the phase transition and the MSE for PCA on $\rdmmat{S}_\kf$ (cyan) and on $\rdmmat{Y}$ (green) is illustrated in both Fig.~\ref{fig:MSE} and Fig.~\ref{fig:Eigengap} for the two examples described in Sec.~\ref{sec:examples}.
	\jump The estimator in Prop.~\ref{cor:FPCA_MSE_performance} based on a PCA of this Fisher matrix, does not use the knowledge of the prior $\pi_\kf$. Since the decomposition in Thm.~\ref{theo:informal_eigvector_decomp_Fisher} corresponds to a Gaussian channel, the Bayes-optimal way to denoise $\rdmvect{v}_1$ alone is to apply the Gaussian denoising function $\eta_{\pi_\kf}$ of Eq.~\eqref{eq:Gauss_denoise}, and we have the following result:
	\begin{cor}[Optimal Denoising of the Top Eigenvector]
		Under the same setting of Thm.~\ref{th:informal_rank1decompositionFisher}, if we denote
		\begin{align}
			\rdmvect{\hat{x}}_{\mathsf{d}}
			:=
			\vect{\eta}_{\pi_\kf} \left( 
			\frac{m_{2 \kf} \mathrm{q}_{0}^2 }{1- \mathrm{q}_{0}^2 } , \frac{\mathrm{q}_{0}}{(1-\mathrm{q}_{0}^2) \sqrt{m_{2 \kf}} } \sqrt{N} \rdmvect{v}_1   \right) \, , 
		\end{align} then we have 
		\label{cor:denoised_FPCA_performance}
		\begin{align}
			\lim_{N \to \infty} \overline{\mathrm{MSE}}_{N} \left( \vect{x}^{\kf} , \rdmvect{\hat{x}}_{\mathsf{d}} \right) 
			&=
			\begin{cases}
				1 - \big(\frac{m_{k_F}}{m_{2 k_F}} \big)^2&\mbox{if} \quad \Delta_{k_F} \geq (m_{2 k_F})^2 \\
				1 - \big(\frac{\mathrm{q}_{1}}{m_{2 k_F}} \big)^2 &\mbox{if} \quad \Delta_{k_F} < (m_{2 k_F})^2 \, ,
			\end{cases}
		\end{align}
		where $\mathrm{q}_{1} \equiv \mathrm{q}_{1}(\Delta_{\kf})$ is given by $\mathrm{q}_{1} =\mathsf{SE}_{\pi_\kf}(\mathrm{q}_0 , \Delta_\kf)$ and we recall that $ \mathrm{q}_0 \equiv \mathrm{q}_0(\Delta_{\kf})$ is given by Eq.~\eqref{eq:def_limoverlap_PCA_Fisher}.
	\end{cor}
	Note that since $\mathsf{SE}_{\pi_\kf}(.,\Delta_\kf)$ is increasing, this  “one-step of denoising" transformation always reduces the MSE of Prop.~\ref{cor:FPCA_MSE_performance}, as illustrated by the purple colored curves in Fig.~\ref{fig:MSE}.

	\paragraph{Approximate Message Passing and State-Evolution --- } Our last set of results concerns the asymptotic behavior of provably optimal algorithms. By the universality of overlaps, we know that samples from the Bayesian model correspond to inference of the output channel encoded by the Fisher matrix. The spiked matrix decomposition of the Fisher matrix further implies that the Fisher matrix is essentially a Wigner spiked matrix. This means we can apply the classical AMP algorithm \cite{rangan2012iterative,deshpande2014information,lesieur2015mmse} to denoise Wigner matrices, up to some control over the error term. Formally, let us define AMP as follows:
	\begin{definition}[AMP]
		Let $\mat{A}$ a $(N \times N)$ symmetric matrix, $f$ a  Lipschitz function from $\mathbb{R}^2 \to \mathbb{R}$, $\vect{f}: \mathbb{R}^{1+N} \to \mathbb{R}^N$ the vectorized version of $f$ for its second argument, $\alpha \in \mathbb{R}$, and $\vect{x}_{ini} \in \mathbb{R}^N$,  then the  \emph{Approximate Message Passing} (AMP in short) iteration is given by the sequence
		\begin{align}
			\label{eq:def_general_AMP}
			\rdmvect{h}_{t+1} =  \mat{A} \rdmvect{\hat{x}}_{t} - \frac{1}{\alpha N} \left(  \nabla_{\rdmvect{h}} \cdot  \,   \vect{f} \left( \frac{ \langle   \rdmvect{\hat{x}}_{t} ,  \rdmvect{\hat{x}}_{t} \rangle }{\alpha N}  , \rdmvect{h}_t \right) \right)  \, \rdmvect{\hat{x}}_{t-1} \qquad ; \quad 
			\rdmvect{\hat{x}}_{t+1} = \vect{f} \left(\frac{ \langle   \rdmvect{\hat{x}}_{t} ,  \rdmvect{\hat{x}}_{t} \rangle }{\alpha N}  ,\rdmvect{h}_t  \right) \, ,
		\end{align}
		with initial conditions $\rdmvect{\hat{x}}_{-1} = \vect{0}$ and $\rdmvect{\hat{x}}_{0} = \vect{x}_{ini}$. In Eq.~\eqref{eq:def_general_AMP}, ‘$\nabla_{\rdmvect{h}} \cdot$' is the divergence operator restricted to the second variable. We  use the notation $\rdmvect{\hat{x}}_{t} \hookleftarrow \mathtt{AMP}_t(\mat{A},f,\alpha;\vect{x}_{ini})$ to denote the  $t$-th iterate of this reccurence.
	\end{definition}	
	\begin{theo}[Universality of Convergence of AMP to State-Evolution]
		\label{theo:AMP_performance}
		Under the same setting of Thm.~\ref{th:informal_rank1decompositionFisher}, if $\rdmvect{v}_1$ the top eigenvector of  $\rdmmat{S}_\kf$, $\eta$ defined by Eq.~\eqref{eq:Gauss_denoise} and $\rdmvect{\hat{x}}_{t} \hookleftarrow \mathtt{AMP}_t( \frac{\rdmmat{S}_\kf}{\sqrt{N}} ,\eta, \Delta_\kf ;\rdmvect{v}_1)$, then we have:
		\begin{align}
			\lim_{t \to \infty} \lim_{N \to \infty} \overline{\mathrm{MSE}}_{N} \left( \vect{x}^{\kf} , \rdmvect{\hat{x}}_{t} \right)= 1 - \left(\frac{\mathrm{q}_{\infty}}{m_{2\kf}} \right)^2 \, ,  
		\end{align}
		where $\mathrm{q}_{\infty}:= \lim_{t \to \infty} \mathrm{q}_t$, and  $\mathrm{q}_t$ satisfies the state-evolution recurrence $\mathrm{q}_{t+1} = \mathsf{SE}_{\pi_\kf}(\mathrm{q}_t,\Delta_\kf) \, , $ with initial condition $\mathrm{q}_{t=0} = \mathrm{q}_0$ defined by Eq.~\eqref{eq:def_limoverlap_PCA_Fisher}.
	\end{theo}
	\begin{proof}
		The proof, that uses crucially Thm.~\ref{th:informal_rank1decompositionFisher}  with the classical state evolution theorems \cite{bayati2011dynamics,berthier2020state,gerbelot2023graph},
		is given in App.~\ref{sec:Proof_AMP}. 
	\end{proof}
	From the Gaussian universality and universality of overlaps, these results show that the AMP estimate for the Bayesian optimal denoising functions can be provably optimal. Indeed, if there is only one fixed point to the equation $\mathrm{q}_{\infty}  = \mathsf{SE}_{\pi_\kf}({\mathrm{q}_{\infty}}_\star,\Delta_\kf)$, then $\mathrm{q}_{\infty}=\mathrm{q}_\star$, where $\mathrm{q}_\star$ is the maximizer of the replica-symmetric function $\mathcal{F}(q,\Delta_\kf)$. This is the case in the examples given in Fig.\ref{fig:MSE}. If, on the other hand, there are multiple fixed points, then AMP may be trip in a local extrema: this situation is the source of the computational to statistical gap in these problems (for more details, see \cite{zdeborova2016statistical,lesieur2017constrained,bandeira2018notes,bandeira2022franz}). In cases where the IT threshold exists and is achievable with AMP, we show that a linearization of the AMP algorithm leads to a spectral algorithm with a phase transition for weak recovery that matches the IT threshold. This follows from the analysis of the state evolution equations with respect to linear denoising functions as in \cite{lesieur2017constrained, Mondelli_spectralamp, Mondelli_spectralamp2} and is detailed in App.~\ref{sec:Linearized_AMP}.
	
	\section{Acknowledgements}

	
	The authors would like to thank Yunzhen Yao for bringing the non-linear problem to their attention and for the initial numerical investigation of the spectrum for the absolute value non-linearity, and Alice Guionnet for insightful discussions.
	We also acknowledge funding from the Swiss National Science Foundation grant SNFS OperaGOST  (grant number $200390$), and SMArtNet (grant number $212049$). J.K. acknowledges the support from the European Research Council (ERC) under the European Union
	Horizon 2020 research and innovation program (grant agreement No. 884584) and the Natural Sciences and Engineering Research Council of Canada (NSERC) (RGPIN-2020-04597).
	
	\newpage  
	
	
	
	\bibliographystyle{amsplain}
	
	\providecommand{\bysame}{\leavevmode\hbox to3em{\hrulefill}\thinspace}
	\providecommand{\MR}{\relax\ifhmode\unskip\space\fi MR }
	\providecommand{\MRhref}[2]{%
		\href{http://www.ams.org/mathscinet-getitem?mr=#1}{#2}
	}
	\providecommand{\href}[2]{#2}

	\newpage
	\appendix
	
	
	\section{Bell Polynomials, Cumulants and Higher Order Fisher's Identities}\label{sec:higherorderfisher}
	
	We introduce the following usual families of polynomials in combinatorics, see \cite{charalambides2018enumerative}.
	\begin{definition}[Bell polynomials, Cumulant polynomials] {}
		\begin{itemize}
			\item The \emph{partial (exponential) Bell polynomials}  $\mathrm{B}_{l,k}$ are defined for any integer $k \in \mathbb{N}_*$ by
			\begin{align}
				\mathrm{B}_{l,k} (x_1, \dots, x_{l-k+1}) := \sum \frac{l!}{j_1! j_2! \cdots j_{l - k + 1}!} \Big( \frac{x_1}{1!} \Big)^{j_1} \cdots \Big( \frac{x_{l-k+1} }{(l - k + 1)!} \Big)^{j_{l - k + 1}} \, ,
			\end{align}
			where the sum is over non-negative integers $j_1, \dots, j_{P - k + 1}$ such that
			\begin{align}
				j_1 + \dots + j_{l - k + 1} = k, \qquad j_1 + 2 j_2 + \dots + (l - k + 1)j_{l - k + 1} = l.
			\end{align}
			\item The \emph{complete (exponential) Bell polynomials}   $\mathrm{B}_l$ are defined as a sum over the second index of the partial ones, namely:
			\begin{align}
				\mathrm{B}_l( x_1, \dots, x_l) :=  \sum_{k = 1}^l \mathrm{B}_{l,k} (x_1, \dots, x_{l - k + 1})  \, . 
			\end{align}
			\item The \emph{cumulant polynomials} $\kappa_{k}$ are defined as: 
			\begin{align}
				\kappa_{k}(x_1,\dots,x_k) &= \sum_{i=1}^k (-1)^{i} (i-1)! \mathrm{B}_{k,i}(x_1,\dots,x_{k-i+1}) \, .
			\end{align}
		\end{itemize} 
	\end{definition}
	Bell (respectively cumulant) polynomials satisfy the recurrence relations: 
	\begin{align}
		\mathrm{B}_{l} = \sum_{i=0}^{l-1} \binom{l}{i} \mathrm{B}_{l-i} x_i \, ,
	\end{align}
	and 
	\begin{align}
		\kappa_k  =  x_{k} - \sum_{i=1}^{k-1} \kappa_{i} x_{k-i}
	\end{align}
	where we dropped the dependency in their arguments for clarity. 
	\jump Our main reason to introduce these polynomials relies on the fact that Bell (respectively cumulant) polynomials appear naturally when considering the composition of the exponential map (resp. the logarithm map) of the expansion of a function, namely, we have two identities: 
	\begin{align}
		\mathrm{B}_l( x_1, \dots, x_l)  &= \frac{\dd^l }{\dd t^l}  \exp \left( \sum_{i=1}^{\infty} \frac{{x_i}^{i}}{i!} t^{i} \right) \bigg |_{t=0} \, , 
	\end{align}
	and
	\begin{align}
		\kappa_{k}(x_1,\dots,x_k) &= \frac{\dd^k }{\dd t^k}  \log \left( 1+ \sum_{i=1}^{\infty} \frac{{x_i}^{i}}{i!} t^{i} \right) \bigg |_{t=0}  \, .
	\end{align}
	such that one can think of them as dual one to another.
	\jump
	In particular, if we apply this duality to the log-likelihood and the output density, we get for any $k$ such that $(\partial_w)^k \pout(y|0)$ is well defined, the following relations  hold:
	\begin{align}
		\frac{(\partial_w^k) \pout(y|0)}{\pout(y|0)}
		&= 
		\mathrm{B}_k \big( g^{(1)}(y,0), \dots, g^{(k)}(y,0))\, , 
	\end{align}
	and
	\begin{align}
		\frac{g^{(k)}(y,0)}{k!} 
		&= 
		\kappa_{k} \left( \frac{ (\partial_w \pout)(y |0)  }{\pout(y|0)} , \dots, \frac{(\partial^{k}_w \pout)(y|0) }{\pout(y|0)} \right) \, .
	\end{align}
	Note that for $ k< \kf$,  $g^{(k)}(y,0) =0$  and thus from the recurrence relation for the cumulants and bell polynomials, we have the simple formula:
	\begin{align}
		\label{eq:gk_to_partial_pout_k_leq_kf}
		\frac{g^{(k)}(y,0)}{k!} 
		&= 
		\frac{(\partial^{k}_w \pout)(y|0) }{\pout(y|0)} \qquad \mbox{for } k \leq \kf \, .
	\end{align}
	Note that only the last case $k = \kf$ in this equation is not the trivial identity "$0=0$". 
	\jump
	This observation is key in the derivation of the higher-order Fisher inequalities, which we now describe and characterize. These identities will be used multiple times throughout the appendices and are the main intuition behind the appearance of the critical Fisher information exponent. 
	\begin{lem}[Higher order Exact and Approximate Fisher's identities]
		\label{lem:ho_identity} 
		If Hypothesis~\ref{hyp:bayes} holds , 
		\begin{itemize}
			\item (Exact Identities)  we have almost surely in $W_{ij}$:
			\begin{align}
				\E\left[ \frac{g^{(k)}(Y_{ij}, W_{ij} )}{k!} \given[\Big] W_{ij} \right] &= 0 \quad \mbox{for} \,  k \in \{ k_F, \dots, 2 k_F -1\} \, ,  \\
				\E \left[ \frac{g^{(2 k_F)}(Y_{ij},  W_{ij} )}{(2 k_F)!} \given[\Big] W_{ij} \right] &= - \frac{1}{2} 	\E \left[ \left( \frac{g^{(k_F)}(Y_{ij},  W_{ij} )}{k_F!} \right)^2 \given[\Big] W_{ij} \right]  \, ,
				\, 
			\end{align}
			\item (Approximate Identities) and if $\beta = \beta_{\mathrm{cr}}$ then we have almost surely in $W_{ij}$:
			\begin{align}
				\E \left[ \frac{g^{(\kf)}(Y_{ij},0)}{\kf!} \given[\Big] W_{ij} \right]  &=  \frac{W_{ij}^{\kf}}{\Delta_{\kf}}  +  O(N^{-\frac{\kf + 1}{2 \kf }})  ,  \\
				\E \left[ \frac{g^{(k)}(Y_{ij},0)}{k!} \given[\Big] W_{ij} \right]  &=   O(N^{-\frac{1}{2}})  , \quad \mbox{for} \,  k \in \{ k_F + 1, \dots, 2 k_F -1\} \, ,  \\
				\E \left[ \frac{g^{(2 k_F )}(Y_{ij},0)}{(2 k_F)!} \given[\Big] W_{ij} \right] &= - \frac{1}{2} \frac{1}{\Delta_{\kf}}  + O(N^{-\frac{1}{2 }}) \, .
				\, 
			\end{align}	
			Furthermore, $ O(N^{-\frac{\kf + 1}{2 \kf }}) \leq O(N^{-1/2})$, which gives a bound independent of $\kf$. 
		\end{itemize}
	\end{lem}
	
	\begin{proof}We begin by proving the exact identities, then use a continuity argument to prove the approximate identities. \hfill\\\\
		\textit{Proof of the Exact Identities:} Firstly, using the fact that $\pout(\cdot|w)$ is a probability measure for all $w$,
		\begin{align}\label{eq:out_probability}
			\int \dd \pout(y|w) = \int e^{ g(y,w)} dy &= 1.
		\end{align}
		By differentiating \eqref{eq:out_probability}, it follows that for all $w$ and all $l \geq 1$,
		\begin{align}
			\int \partial_{w}^{l} e^{ g(y,w)} \, dy  &= 0.
		\end{align}
		The derivatives are given by:
		\begin{align}\label{eq:derivg}
			\partial_{w}^{l} e^{ g(y,w)} = e^{ g(y,w)} \mathrm{B}_l(g^{(1)}, \dots, g^{(l)})
		\end{align}
		where $\mathrm{B}_l$ is the Bell polynomial of the previous paragraph. 
		Furthermore, by Hypothesis~\ref{hyp:bayes} it is immediate that
		\begin{align}
			\int \partial_{w}^{k_F} e^{ g(y,w)} \, dy  = 0 \implies \E_{Y|w}  g^{(\kf)}(Y,w) = 0
		\end{align}
		because all $ g^{(k)}(Y,w) \equiv 0$ for all $k < k_F$.  Likewise, for $l \in \{k_F + 1, \dots, 2 k_F - 1 \}$ it follows that
		\begin{align}
			\int \partial_{w}^{l} e^{ g(y,w)} \, dy  = 0 \implies \E_{Y|w}  g^{(l)}(Y,w) = 0
		\end{align}
		because all cross terms in the Bell polynomials vanish because $ g^{(k)}(Y,w) \equiv 0$ for all $k < k_F$. Lastly, it follows that 
		\begin{align}
			\int \partial_{w}^{2 k_F} e^{ g(y,w)} \, \dd y  = 0 \implies \E_{Y|w}  g^{(2 k_F)}(Y,w) =  - \frac{(2 k_F)!}{ 2 ( k_F!)^2 } \E_{Y|w}  (g^{( k_F)}(Y,w))^2
		\end{align}
		again because all cross terms in the Bell polynomials vanish except for the term corresponding to $(g^{(k_F)})^2$ and $g^{(2k_F)}$. 
		\\\\
		\textit{Proof of Approximate Identities:} For this proof, we assume that $\beta = \beta_{\mathrm{cr}} =\frac{1}{2} \left(1- \frac{1}{k_F} \right)$ defined in  \eqref{eq:criticalscaling} so 
		\[
		\mat{W} 
		=
		(W_{ij})_{1 \leq i,j \leq N} = \frac{1}{N^{\frac{1}{2\kf}}}\, 
		\vect{x} {\vect{x}^{\top}} \implies O( |W_{ij}|_{\infty} ) = O(N^{-\frac{1}{2 \kf}})
		\]
		since we assumed that $\pi$ has compact support. The assumptions in Hypothesis~\ref{hyp:bayes} are precisely what is required to Taylor expand $g$ up to the $\kf$ order and bound the remainder term.
		
		We first compute
		\[
		\E \Big[ \frac{g^{(\kf)}(Y_{ij},0)}{\kf!} \given[\Big] W_{ij}  \Big]  = \frac{1}{\kf!} \int  g^{(\kf)}(y,0)  e^{g(y,W_{ij})} \, dy.
		\]
		We can Taylor expand $ e^{g(y,W_{ij})}$ up to the $\kf$th order and notice that all subleading terms vanish because the corresponding Bell polynomial is $0$ 	because all $ g^{(k)}(Y,w) \equiv 0$ for all $k < k_F$ to conclude that 
		\begin{align*}
			\frac{1}{\kf!}\int  g^{(\kf)}(y,0)  e^{g(y,W_{ij})} \, dy &= \frac{1}{(\kf!)^2} \int ( g^{(\kf)}(y,0) )^2 e^{ g(y,0)} \, dy W_{ij}^\kf +O( |W_{ij}|_{\infty}^{\kf + 1} ) 
			\\&= \frac{1}{\Delta_{\kf}} W_{ij}^\kf + O( |W_{ij}|_{\infty}^{\kf + 1} ) =  \frac{1}{\Delta_{\kf}} W_{ij}^\kf  +  O(N^{-\frac{\kf + 1}{2 \kf }})
		\end{align*}
		since again, we only have to keep the terms with $g^{(\kf)}(y,0)$ appearing by itself since all other terms vanish as in the proof of the exact identities. 
		
		Next we compute 
		\[
		\E \left[ \frac{g^{(2 k_F )}(Y_{ij},0)}{(2 k_F)!} \given[\Big] W_{ij} \right] =   \int  \frac{g^{(2 k_F )}(Y_{ij},0)}{(2 k_F)!}  e^{g(y,W_{ij})} \, dy.
		\]
		We can take the expansion of the exponential up to the $\kf$ order since $ g^{(k)}(Y,w) \equiv 0$ for all $k < k_F$. The constant order term is non-zero by the second exact identities in Lemma~\ref{lem:ho_identity},
		\begin{align*}
			\int  \frac{g^{(2 k_F )}(y,0)}{(2 k_F)!}  e^{g(y,W_{ij})} \, \mathrm{d}y =  \int  \frac{g^{(2 k_F )}(y,0)}{(2 k_F)!}  e^{g(y,0)} \, \mathrm{d}y + O( |W_{ij}|^\kf_{\infty} ) = - \frac{1}{2} \frac{1}{\Delta_{\kf}}  + O( N^{-\frac{1}{2}} ).
		\end{align*}
		The same computation implies that for all $k \in \{ \kf + 1, \dots, 2\kf - 1 \}$ that
		\[
		\E \left[ \frac{g^{(k )}(Y_{ij},0)}{k!} \given[\Big] W_{ij} \right] =   \int  \frac{g^{(k )}(Y_{ij},0)}{k!}  e^{g(y,W_{ij})} \, dy.
		\]
		We can take the expansion of the exponential up to the $\kf$ order since $ g^{(k)}(Y,w) \equiv 0$ for all $k < k_F$. The constant order term is zero by the second exact identities in Lemma~\ref{lem:ho_identity},
		\begin{align*}
			\int  \frac{g^{(k )}(y,0)}{k!}  e^{g(y,W_{ij})} \, \mathrm{d}y =  \int  \frac{g^{(k)}(y,0)}{k!}  e^{g(y,0)} \, \mathrm{d}y + O( |W_{ij}|^\kf_{\infty} ) = O( N^{-\frac{1}{2}} )).
		\end{align*}
	\end{proof}
	
	\newpage
	\section{Proof of Universality of the Free Energy (Th.~\ref{theo:Gauss_Approximation_Free_Energy} and Th.~\ref{theo:Gauss_Universality})}
	\label{sec:Proof_Universality}

	\subsection{Outline of the Proof}
	
	In this section, we prove the universality of the free energy. We begin by providing an overview of the main steps in the proof. Throughout this entire section, we assume that $\beta_c$ is given by the critical scaling \eqref{eq:criticalscaling}.
	
	Our starting point is the free energy of the general inference problem
	\begin{align}
		F_N(g,\pi) 
		&:=
		\frac{1}{N} \bigg( \E \log \int e^{ \sum_{i < j} g(Y_{ij},M_{ij}) } \,  \dprior^{\otimes N}( \vect{m} ) - \sum_{i < j} \E g(Y_{ij},0) \bigg).
	\end{align}
	where we denote the dummy variable at the critical scaling
	\begin{equation}\label{eq:fisherscores2}
		\mat{M} = \frac{1}{N^{\frac{1}{2\kf}}}	\vect{m} {\vect{m}^{\top}} \quad \text{ and } \qquad \mat{M}_{\kf} =	\frac{1}{\sqrt{N}} \vect{m^\kf} {\vect{m^{\kf}}^{\top}}
	\end{equation}
	to not confuse it with the signal matrix $\rdmmat{W}$. 
	Notice that the sum is over $i < j$, but it can be replaced by $i \leq j$ without changing the limit of the free energy because including the diagonals will introduce an order $O(\frac{1}{N})$ error, which is negligible in the limit.
	
	To prove the universality of the overlaps, we also need to prove an equivalent statement for the constrained free energy, that is given sets $A, B \subseteq \R$, we define
	\begin{align}\label{eq:freeenergybayes_constrained}
		F_N(A,B; g,\pi) 
		&:=
		\frac{1}{N} \bigg( \E \log \int \1(R^{\kf}_{10} \in A, R^{\kf}_{11} \in B) e^{ \sum_{i < j} g(Y_{ij},M_{ij}) } \,  \dprior^{\otimes N}( \vect{m} )  - \sum_{i < j} \E g(Y_{ij},0)  \bigg),
	\end{align}
	where the we denote the overlaps by $R^{\kf}_{10}$ and $R^{\kf}_{11}$ denote the normalized innner products
	\begin{equation}\label{eq:defnOverlaps}
		R^{\kf}_{10} = \frac{\langle \vect{m}^\kf, \vect{x}^\kf \rangle}{N}  \qquad \text{ and } \qquad R^{\kf}_{11} =  \frac{\langle \vect{m}^\kf, \vect{m}^\kf \rangle }{N} .
	\end{equation}
	We will soon see that the following proofs hold for both $F_N(g,\pi)$ and $	F_N(A,B; g,\pi) $, so for notation simplicity, we consider only the classical free energy $F_N(g,\pi)$. One can interpret the constrained free energy as a restriction of $\dprior^{\otimes N}$ to the sets $ \1(R^{\kf}_{10} \in A, R^{\kf}_{11} \in B)$, and it is easy to spot that the following proof of universality holds for general reference measures that do not need to be probability measures, nor do they have to be a product measure. The product measure assumption is however essential to compute the limit in Corollary~\ref{cor:limiting_value_Free_Energy}.
	
	We first recall the definition of the leading order approximation of the free energy,
	\begin{align}
		F^{(\mathrm{ap})}_N \left( g^{(k_F)}(.,0) , \pi_{k_F} \right)  := \frac{1}{N} \bigg( \E \log \int e^{\frac{1}{2} \Tr(\rdmmat{S}_{k_F}  \mat{M}_{k_F} ) -  \frac{1}{4 \Delta} \Tr(  \rdmmat{M}_{k_F} \rdmmat{M}_{k_F})  } \, \dpriortensor( \vect{m} )
		\Big) - \sum_{i < j} \E g(Y_{ij},0) \bigg).
	\end{align}
	Likewise, we define $F^{A,B; (\mathrm{ap})}_N \left( g^{(k_F)}(.,0) , \pi_{k_F} \right)$ by restricting the overlaps $R_{10}^\kf$ and $R_{10}^\kf$ analogously to \eqref{eq:freeenergybayes_constrained}.
	This quantity completely encodes the fundamental limits of the inference problems in the sense that these free energies have the same limit.

	\begin{lem}[Leading Order Approximation of the Free Energy]\label{lem:1st_approx_app}
		Under Hypothesis~\ref{hyp:bayes} and the critical scaling of Eq.~\eqref{eq:criticalscaling}, we have
		\[
		F_N(g , \pi) = F^{(\mathrm{ap})}_N \left( g^{(k_F)}(.,0) , \pi_{k_F} \right) + O( N ^{- \frac{1}{2 k_F}}) .
		\]
		Similarly, for any $A$ and $B$ 
		\[
		F_N(A,B ; g , \pi) = F^{(\mathrm{ap})}_N \left(A,B; g^{(k_F)}(.,0) , \pi_{k_F} \right) + O( N ^{- \frac{1}{2 k_F}}) .
		\]
	\end{lem}
	
	Next, we can show that the approximation the free energy is equivalent to the free energy of a higher-order Gaussian model,
	\[
	F^{(G)}_N(\Delta_{k_F} , \pi_{\kf}) =  \frac{1}{N}  \E \log \int e^{\frac{1}{2 \Delta} \Tr( \sqrt{\Delta} \, \rdmmat{G} + \mat{W}_{\kf} ) \mat{M_\kf} ) -  \frac{1}{4 \Delta} \Tr(  \mat{M_\kf} \mat{M_\kf})  } \, \mathrm{d} \pi^{\otimes N}( \vect{m}  ).
	\] 
	The constrained analogue $F^{(G)}_N(A,B; \Delta_{k_F} , \pi_{\kf}) $ is defined by restricting the overlaps $R_{10}^\kf$ and $R_{10}^\kf$ analogously to \eqref{eq:freeenergybayes_constrained}.
	Notice that there is an extra $\frac{1}{2}$ factor in front of the trace, because we are summing over $i < j$ and $i > j$, so we have to normalize this by two to remain consistent with the convention in \eqref{eq:freeenergybayes} and \eqref{eq:freeenergybayes_constrained}. 
	
	\begin{lem}[Gaussian Approximation of the Free Energy]\label{lem:2nd_approx}
		If  $\frac{g^{(\kf)}(Y_{ij},0)}{(\kf)!}$ has bounded third moment, then
		\begin{equation}
			F^{(\mathrm{ap})}_N \left( g^{(k_F)}(.,0) , \pi_{k_F} \right) = F^{(G)}_N \left(\Delta_{k_F} , \pi_{k_F} \right)  + O(N^{-1/2}).
		\end{equation}	    
		Similarly, for any $A$ and $B$ 
		\begin{equation}
			F^{(\mathrm{ap})}_N \left(A,B; g^{(k_F)}(.,0) , \pi_{k_F} \right) = F^{(G)}_N \left(A,B; \Delta_{k_F} , \pi_{k_F} \right)  + O(N^{-1/2}).
		\end{equation}	    
	\end{lem}
	This concludes the Gaussian universality Theorem~\ref{theo:Gauss_Approximation_Free_Energy}.
	In the following subsections, we prove these first and second approximations of the free energy, and respectively the constrained free energy.

	\subsection{Proof of the Leading Order Approximation (Lemma \ref{lem:1st_approx_app})}
	let us first define the function
	\begin{align}
		g_{\mathrm{trunc}}(y,w)
		:=  
		g(y,0) + \left[ \sum_{k=k_F}^{2 k_F} \frac{g^{(k)}(y,0)}{k!}  \, w^{k}  \right].
	\end{align}
	By definition $W_{ij} =O(N^{\beta_c-1/2}) = O(N^{\frac{1}{2\kf}})$ and it is small at large $N$ so a Taylor expansion around $W_{ij}=0$ gives us the following.
	\begin{lem}[Truncation Approximation]
		\label{lem:univ1}
		Under Hypothesis~\ref{hyp:bayes} and the critical scaling $\beta_{\mathrm{cr}}$ defined in \eqref{eq:criticalscaling}
		\begin{align}
			F_N(g , \pi) = F_N( g_{\mathrm{trunc}} , \pi) + O \Big( N ^{- \frac{1}{2\kf}} \Big) 
		\end{align}
	\end{lem}
	
	\begin{proof}
		By Taylor's theorem, for all $i,j$,
		\begin{align}
			g(  Y_{ij},M_{ij})-g( Y_{ij},0)= & \partial_w g( Y_{ij},0) M_{ij} +\left[ \sum_{k= k_F}^{2 k_F} \frac{g^{(k)}(Y_{ij},0)}{k!}  \, M_{ij}^{k}  \right]  \\
			&+\frac{M_{ij}^{2 k_F + 1}}{(2 k_F+ 1)!} g^{(2 k_F + 1)}(Y_{ij},\theta_{ij} M_{ij}) \, ,
		\end{align}
		for some $\theta_{ij}\in [0,1]$. Since our hypothesis implies that $|M_{ij}|_\infty \le C^2 N^{-1/2k_F}$,
		we have
		\begin{align}
			|M_{ij}^{(2 k_F + 1)}|_\infty \leq C^2 N^{- 1 - \frac{1}{2 \kf}}.
		\end{align}
		Our assumption that $\|g^{2 k_F + 1}\|_\infty < \infty$ implies that
		\begin{align}
			\bigg|\frac{1}{N} \sum_{i < j} \frac{M_{ij}^{(2 k_F + 1)}}{(2 k_F + 1)!}g^{(2 k_F + 1)}(Y_{ij},\theta_{ij}M_{ij} ) \bigg| =  O\left(N^{-\frac{1}{2 k_F}}\right) .
		\end{align}
		from which the result follows.
	\end{proof}
	
	Next, we use concentration to show that we can replace the higher-order terms with their expected value. Let $B$ be any set such that $Y_{ij}$ is independent conditionally on $B$. We define
	\begin{align}
		\overline{{g}_{\mathrm{trunc}}} (Y_{ij},M_{ij})
		:=  
		g(Y_{ij},0) + \frac{g^{(k_F)}(Y_{ij},0)}{k_F!}  \, M_{ij}^{k_F} + \left[ \sum_{k=k_F + 1}^{2k_F} \E \Big[ \frac{g^{(k)}(Y_{ij},0)}{k!} \given[\Big] B \Big]  \, M_{ij}^{k}  \right].
	\end{align}
	
	\begin{lem}[Concentration of Higher Order Terms] \label{lem:univ2} 
		Assume $ \|\partial_w^{l} g(\cdot ,0) \|_\infty < \infty$ for $l \in \{\kf + 1, \dots, 2 \kf \}$.  
		Let $B$ be a $\sigma$ algebra  such that
		the $Y_{ij}$ are independent conditionally to $B$.
		Then 
		\begin{align}\label{eq:diff_FE_trunc_conditional_trunc}
			F_N( g_{\mathrm{trunc}}, \pi )=  F_N(\overline{{g}_{\mathrm{trunc}}}, \pi) +O(N^{ -\frac{1}{2\kf}}   ) \, . 
		\end{align}
	\end{lem}
	
	\begin{proof}
		For $l \in \{ k_F + 1, \dots, 2 k_F \}$ we define
		\begin{align}
			\bar g_{l}(Y_{ij},M_{ij}) = g(Y_{ij},0) + \left[ \sum_{k=k_F}^{l - 1} \frac{g^{(k)}(Y_{ij},0)}{k!}  \, M_{ij}^{k}  \right] + \left[ \sum_{k=l}^{2 k_F} \E \Big[ \frac{g^{(k)}(Y_{ij},0)}{k!} \given[\Big] B \Big]  \, M_{ij}^{k}  \right].
		\end{align}
		
		Notice that $ \overline{{g}_{\mathrm{trunc}}}  = \bar g_{\kf + 1}$ and we have the following identity:
		\begin{align}
			F_N( g_{\mathrm{trunc}}, \pi) - F_N( \overline{{g}_{\mathrm{trunc}}} , \pi) = \sum_{l= k_F + 1}^{2 k_F}  F_N( \bar g_{l+1}, \pi ) - F_N( \tilde g_{l}, \pi ) .
		\end{align}
		Our goal is to control each of the increments for $l \in \{k_F + 1, \dots, 2 k_F \}$ by proving the following bound:
		\begin{align}\label{eq:differencebound}
			F_N( \bar g_{l + 1} , \pi ) - F_N( \bar g_{l} , \pi ) = O( N^{- \frac{l }{2k_F} + \frac{1}{2} } ) \leq O( N^{- \frac{(k_F+1)}{2k_F} + \frac{1}{2} }  ) = O(N^{-\frac{1}{2\kf}}) , 
		\end{align}
		which in turn, clearly implies the desired result for the difference of Eq.~\eqref{eq:diff_FE_trunc_conditional_trunc}.
		
		Notice that
		\begin{align}
			F_N( \bar g_{l + 1} , \pi ) - F_N( \bar g_{l} , \pi ) =\E\frac{1}{N}\log \Big\langle e^{\frac{1}{l!}\sum_{i< j} (g^{(l)}(Y_{ij},0)-\E[  g^{(l)}(Y_{ij},0)|B])M^l_{ij})}\Big\rangle
		\end{align}
		where the average $\langle \cdot \rangle$ is defined as:
		\begin{align}
			\langle f\rangle := \frac{\int f e^{\sum_{i< j}  \tilde g_{l} (Y_{ij}, M_{ij}) } \, \dprior^{\otimes N}(\vect{m})}{ \int e^{\sum_{i< j}  \tilde g_{l} (Y_{ij}, M_{ij}) } \, \dprior^{\otimes N}(\vect{m})}\,.
		\end{align}
		Recall that
		\begin{align}
			M_{ij}^ l =  \frac{m_i^l m_j^l}{N^{\frac{l }{2\kf}}}
		\end{align}
		and since $\kf \geq 1$ and $l > \kf$, we have $\frac{l }{2\kf} > \frac{1}{2}$. 
		
		Let $\rdmmat{Z}_l = (Z_{ij,l})_{1 \leq i,j \leq N}$ be the $N\times N$ symmetric  matrix with entries such that:
		\begin{itemize}
			\item diagonal terms are zero: $Z_{ii ,l}=0$ for $i = \{ 1, \dots, N\}$
			\item and off-diagonal terms are given by $  Z_{ij,l} =\frac{1}{2 l! \sqrt{N}}( g^{(l)}(Y_{ij},0)-\E[ g^{(l)}(Y_{ij},0)|B])$, 
		\end{itemize}
		such that we can write the  following sum in terms of this matrix $\rdmmat{Z}_l$:
		\begin{align}
			\sum_{i< j} \frac{1}{2 l! \sqrt{N}}(g^{(l)}(Y_{ij},0)-\E[ g^{(l)}(Y_{ij},0)|B])\Big(\frac{m_i^l m_j^l}{N^{\frac{l }{2\kf} - \frac{1}{2}} }\Big) =\Tr\left ( \rdmmat{Z}_l \mat{M}_{l} \right).
		\end{align}
		$\rdmmat{Z}_l$ is a random  symmetric matrix under $\pP_B$ which has centered independent entries with covariance bounded by $C/N$ and where $\mat{M}_l$ is the matrix with entries $\frac{x_i^l x_j^l}{N^{\frac{l }{2k_F} - \frac{1}{2}} }$ . 
		
		Because $\sqrt{N} \rdmmat{Z}_l$  has  bounded entries, we can use concentration inequalities due to Talagrand (see \cite[Theorem 2.3.5]{AGZ}  and \cite[Lemma 5.6]{HuGu}) to see that there exists some finite $L_0$ such that, uniformly, 
		\begin{equation}\label{conc2}
			\pP_B\left(\normop{\rdmmat{Z}_l} \ge L\right)\le e^{-N(L-L_0)}\,.
		\end{equation}
		On $\{\normop{\rdmmat{Z}_l}\le L\}$, we have the bound
		\begin{align}
			\left|\Tr\left (\rdmmat{Z}_l \mat{W}_l \right) \right|=\left| \sum_{i,j} Z_{ij,l} \frac{x_i^l x_j^l}{N^{\frac{l}{2\kf} - \frac{1}{2}} } \right| \le \frac{L}{N^{\frac{l}{2 \kf} + \frac{1}{2} }}  \sum_{i=1}^N x_i^{2l} \le C^{2l} L N^{1 - \frac{l}{2k_F} + \frac{1}{2} } 
		\end{align}
		for some finite constant $C$ depending only on the bound on the support of $\prior$. In particular, we have for $l \in \{k_F + 1, \dots, 2 k_F \}$ that  $- \frac{l }{2k_F} + \frac{1}{2} < 0$, so
		\begin{align}
			\E \Ind{\normop{\rdmmat{Z}_l}\le L}  \frac{1}{N}\log \Big\langle e^{\frac{1}{l!}\sum_{i< j} (g^{(l)}(Y_{ij},0)-\E[  g^{(l)}(Y_{ij},0)|B])W^l_{ij})}\Big\rangle = O( N^{- \frac{l }{2k_F} + \frac{1}{2} } ).
		\end{align}
		On the other hand, we have
		\begin{align}
			\E \Ind{\normop{\rdmmat{Z}_l}\ge L}  \frac{1}{N}\log \Big\langle e^{\frac{1}{l!}\sum_{i< j} (g^{(l)}(Y_{ij},0)-\E[  g^{(l)}(Y_{ij},0)|B])W^l_{ij})}\Big\rangle 
		\end{align}
		is uniformly bounded because we assumed that  $g^{l}(Y,0)$ is uniformly bounded and is going to zero exponentially fast by \eqref{conc2} for $L$ large enough. We conclude that
		\begin{align}
			F_N( \bar g_{l + 1} , \pi ) - F_N( \bar g_l, \pi ) = O( N^{- \frac{l }{2\kf} + \frac{1}{2} } )
		\end{align}
		which completes the proof of equation \eqref{eq:differencebound} and therefore the proof of the Lemma.
	\end{proof}
	
	In applications, we take the conditioning set $B = \rdmmat{W}$. The $Y_{ij}$ are clearly independent under this conditioning because of the form of the nonlinearity \eqref{eq:notation_rank1_mat}.
	
	\begin{lem}[Conditional truncation to Leading order Approximation]\label{lem:trunc_to_lo}
		Under Hypothesis~\ref{hyp:bayes} and the critical scaling of Eq.~\eqref{eq:criticalscaling}, we have
		\[
		F_N(\overline{{g}_{\mathrm{trunc}}}, \pi ) = F^{(\mathrm{ap})}_N \left( g^{(k_F)}(.,0) , \pi_{k_F} \right) + O(N^{-\frac{1}{2\kf}}) .
		\]
	\end{lem}
	
	\begin{proof}
		Recalling the notation of $S$ and $M$ in \eqref{eq:fisherscores} and \eqref{eq:fisherscores2}, Lemma~\ref{lem:ho_identity}, implies that
		\begin{align*}
			\bar g(Y_{ij},M_{ij})
			&=  
			g(Y_{ij},0) + \frac{g^{(k_F)}(Y_{ij},0)}{k_F!}  \, M_{ij}^{k_F} + \left[ \sum_{k=k_F + 1}^{2 k_F} \E \Big[ \frac{g^{(k)}(Y_{ij},0)}{k!} \given[\Big] W_{ij} \Big]  \, M_{ij}^{k}  \right]
			\\
			&= 
			g(Y_{ij},0) + \frac{g^{(k_F)}(Y_{ij},0)}{k_F!}  \,\frac{m_i^{k_F} m_j^{k_F}}{\sqrt{N}} + \left[ \sum_{k=k_F + 1}^{2 k_F} \E \Big[ \frac{g^{(k)}(Y_{ij},0)}{k!} \given[\Big] W_{ij} \Big]  \, \frac{m_i^{k} m_j^k}{N^{\frac{k}{2 k_F}}}   \right] .
		\end{align*}
		Notice that for $k \geq k_F + 1$, we have
		\begin{equation}\label{eq:boundM}
			\frac{m_i^{k} m_j^k}{N^{\frac{k}{2 k_F}}} \leq O\Big( \frac{1}{N^{\frac{1}{2} + \frac{1}{2 k_F}}} \Big).
		\end{equation}
		Recalling the notation of $S$ and $M$ in \eqref{eq:fisherscores} and \eqref{eq:fisherscores2}, the approximate Fisher identities in Lemma~\ref{lem:ho_identity}, implies that 
		\begin{align*}
			\bar g(Y_{ij},M_{ij})
			&= g(Y_{ij},0) + \frac{g^{(k_F)}(Y_{ij},0)}{\kf!}\frac{m_i^{k_F} m_j^{k_F}}{\sqrt{N}} - \frac{1}{2 \Delta} \frac{m_i^{2 k_F} m_j^{2 k_F}}{N}  + O(N^{-1 - \frac{1}{2 k_F}})
			\\&= g(Y_{ij},0) + S_{ij} (M_\kf)_{ij} - \frac{1}{2 \Delta} (M_\kf)_{ij}^2  + O(N^{-1 - \frac{1}{2\kf}}).
		\end{align*}
		The $O(N^{-1 - \frac{1}{2\kf}})$ error term comes from the fact that the conditional expected value has errors of order $O(N^{-\frac{1}{2}})$ by the approximate Fisher identities in Lemma~\ref{lem:ho_identity} and \eqref{eq:boundM}.
		Therefore, summing over $i < j$ implies
		\[
		\sum_{i < j} \bar g(Y_{ij},W_{ij})  =  \sum_{i < j} \bigg[ g(Y_{ij},0) + S_{ij} (M_\kf)_{ij} - \frac{1}{2 \Delta} (M_\kf)_{ij}^2 \bigg] + O(N^{1 - \frac{1}{2\kf}}).
		\]
		The required result is now immediate from the free energy approximations in Lemma~\ref{lem:univ1}, Lemma~\ref{lem:univ2} and the fact that there is a $\frac{1}{N}$ outside of the logarithm in \eqref{eq:freeenergybayes}.
	\end{proof}

	\subsection{Proof of the Gaussian Approximation (Lemma \ref{lem:2nd_approx})}
	\label{sec:Proof_Gauss_approx}
	
	To do computations with $F_N(\rdmmat{S}_{k_F},\rdmmat{M}_{k_F})$, we need to do one more step in the universality and use the fact that we can approximate the matrix $S$ with a Gaussian matrix with the right mean
	\[
	\rdmmat{S}_{k_F} \approx \frac{1}{\sqrt{\Delta_{k_F} }} \rdmmat{G} + \frac{1}{\Delta_{k_F}} \mat{W}.
	\]
	This proof is standard, and follows from the classical Gaussian disorder universality.
	This implies that we can study the free energy
	\begin{equation}\label{eq:FEGaussian}
		F^{(G)}_N(\Delta_{k_F} , \pi) =  \frac{1}{N}  \E \log \int e^{\frac{1}{2 \Delta} \Tr( \sqrt{\Delta} \, \rdmmat{G} + \rdmmat{W}_{\kf} ) \rdmmat{M_\kf} ) -  \frac{1}{4 \Delta} \Tr(  \rdmmat{M_\kf} \rdmmat{M_\kf})  } \, \pi^{\otimes N}( \vect{m}  ).
	\end{equation}
	\begin{theo}
		If  $\frac{g^{(\kf)}(Y_{ij},0)}{(\kf)!}$ has bounded third moment, then
		\begin{equation}
			| F^{(\mathrm{ap})}_N \left( g^{(k_F)}(.,0) , \pi_{k_F} \right) - \tilde F_N(\Delta_{k_F}) | \leq O(N^{-1/2}).
		\end{equation}
	\end{theo}
	\begin{proof}
		Recall that
		\[
		\rdmmat{S} := \left( S_{ij} = \frac{ g^{(\kf)}(Y_{ij},0)}{\kf!} \right)_{1 \leq i,j \leq N} \quad \mbox{and} \quad \frac{1}{\Delta} 
		:= \E_{Y_{ij}|0} \left[ \left( \frac{ g^{(\kf)}(Y_{ij},0)}{\kf!}\right)^2 \right] \, .
		\]
		Let us denote the mean and variance of $\rdmmat{S}$ with
		\[
		\mu_{ij} = \E \Big[ \frac{g^{(\kf)}(Y_{ij},0)}{\kf!} \given W_{ij} \Big] \qquad \sigma^2_{ij} = \E\bigg[ [ S_{ij} - \mu_{ij}]^2 \given W_{ij} \bigg] 
		\]
		and $\rdmmat{\sigma^2}$ and $\rdmmat{\mu}$ be the associated matrices. 
		
		It remains to show taht we can replace the disorder with Gaussian random variables. We follow the argument of \cite{UnivCARMONAHU} as presented in \cite[Section~3.8]{PBook}. If $\frac{g^{(\kf)}(Y_{ij},0)}{\kf!}$ has bounded third moment, then the universality in disorder for spin glasses implies that
		\[
		\tilde F_N(\mat{\sigma},\mat{\mu}) = \frac{1}{N} \bigg( \E \log \int e^{ \Tr((\mat{\sigma} \odot \rdmmat{G} ) \rdmmat{M_\kf} ) +  \Tr(\rdmmat{\mu}  \rdmmat{M_\kf} ) -  \frac{1}{4 \Delta} \Tr(  \rdmmat{M_\kf} \rdmmat{M_\kf})  } \, \dprior( \vect{m} )
		\Big) \bigg)
		\]
		satisfies 
		\[
		| F^{(\mathrm{ap})}_N \left( g^{(k_F)}(.,0) , \pi_{k_F} \right) -\tilde F_N(\mat{\sigma},\mat{\mu})  | \leq O(N^{-1/2})
		\]
		Next, by the approximate Fisher identities in Lemma~\ref{lem:ho_identity}, it follows that
		\[
		\mu_{ij} = \frac{1}{\Delta_\kf} W^\kf_{ij} + O(N^{-\frac{\kf + 1}{2 \kf }}),
		\]
		and
		\[
		\sigma^2_{ij} = \frac{1}{\Delta_\kf} + O(N^{-1/2}).
		\]
		Therefore, we can replace the $\mu$ and $\sigma$ with these quantities by a standard interpolation argument, see for example \cite[Lemma~3.7]{AJFL_inhomo}. It is worth emphasizing that the derivative of the interpolating free energy for the $\mu$ approximation terms will introduce error terms of the order
		\[
		O( N^{\frac{1}{2}} \cdot N^{-\frac{\kf + 1}{2 \kf }}) = O(N^{\frac{1}{2\kf}}) \to 0.
		\]
		because $M_{ij}^\kf = O(N^{-\frac{1}{2}})$ and the $\sigma^2$ approximation terms will introduce terrors of the order $O(N^{-1/2}) \to 0$.
	\end{proof}

	\subsection{Proof of the Universality of the Overlaps (Lemma \ref{lem:2nd_approx})}
	\label{sec:Proof_Universality_Overlaps}
	
	By Lemma~\ref{lem:1st_approx_app} and Lemma~\ref{lem:2nd_approx}, we have concluded that 
	\begin{equation}
		F_N( g , \pi) = F^{(G)}_N \left( \Delta_{k_F} , \pi_{k_F} \right)  + O(N^{-1/2})
	\end{equation}	 
	and 
	\begin{equation}
		F_N(A,B; g , \pi) = F^{(G)}_N \left(A,B; \Delta_{k_F} , \pi_{k_F} \right)  + O(N^{-1/2}).
	\end{equation}
	By taking the difference of these equations, we can conclude that
	\begin{equation}\label{eq:universality_overlaps}
		\frac{1}{N} \E \log \langle \1( R^{\odot \kf}_{10} \in A , R^{\odot \kf}_{11} \in B ) \rangle = 	\frac{1}{N} \E \log \langle \1( R^{\odot \kf}_{10} \in A , R^{\odot \kf}_{11} \in B ) \rangle_G
	\end{equation}
	where the first average is over $\langle \cdot \rangle$ the posterior average
	\[
	\langle f \rangle = \frac{\E \log \int f(\vect{y}) e^{ \sum_{i < j} g(Y_{ij},M_{ij}) } \,  \dprior^{\otimes N}( \vect{y}) }{Z}
	\]
	and the second average $\langle \cdot \rangle_G$ is over the Gaussian equivalent
	\[
	\langle f \rangle_G =  \frac{\E \log \int f(\vect{y}) e^{  \frac{1}{2 \sqrt{\Delta_{k_F}} } \Tr(\rdmmat{G}  \rdmmat{M}_{k_F} ) + \frac{1}{2 \Delta_{k_F}} \Tr(\rdmmat{M_0}  \rdmmat{M}_{k_F} ) -  \frac{1}{4 \Delta_{k_F}} \Tr(  \rdmmat{M}_{k_F} \rdmmat{M}_{k_F}) ) } \,  \dprior^{\otimes N}( \vect{y})  }{Z_G}
	\]
	This implies that in the limit, the quenched overlaps are characterized by the behaviors of the overlaps under the equivalent Gaussian model. We will see in the next section that there is an explicit rate function for the Gaussian model, and that the quenched large deviations principle can be improved to an almost sure one. 
	
	From the large deviations principle of the overlaps, it is possible to get the average MMSE between the signal vector $\rdmvect{x}^\kf$ and the optimal Bayes estimator. This follows from the fact that the MMSE is completely characterized by the overlaps $R_{10}^{\odot \kf}$ and $R_{11}^{\odot \kf}$ by an application of the Nishimori identity. Let $\langle \cdot \rangle = \E [ \cdot \given \rdmmat{Y} ]$ denote its average with respect to the posterior, and let $\rdmvect{y}$ be a sample from the posterior, we have
	\begin{align}
		\overline{\mathrm{MSE}}_{N}( \vect{x}^{\kf}, \rdmvect{y}^{\kf} )   = \frac{  \E \langle
			\|   \vect{x}^{\kf} (\vect{x}^{\kf})^{\top}  -  \rdmvect{y}^{\kf} (\rdmvect{y}^{\kf})^{\top}  \|^2_F \rangle}{  \E \| \vect{x}^{\kf}(\vect{x}^{\kf})^{\top}  \|^2_F} 
		= 1 - \bigg( \frac{ \E \langle \vect{x}^\kf, \vect{y}^\kf \rangle}{ m_{2\kf} } \bigg)^2. \label{eq:MSEoverlaps}
	\end{align}
	since
	\[
	\E [ \E[ \vect{x}^\kf | \rdmmat{Y} ]^2 ] = \E [ \vect{x}^\kf \E[ \vect{x}^\kf | \rdmmat{Y} ] ]
	\]
	by the Nishimori identity.
	In particular, the limiting averaged mean squared error is characterized by the limiting behavior of the overlaps. In the following section, we will characterize this limiting value and show that the right hand side of \eqref{eq:MSEoverlaps} is given by the maximizer of the critical fixed point equation \eqref{eq:qstar} with prior $\pi_\kf$.
	\newpage
	\section{Proof of the IT Thresholds and Optimal MMSE (Th.~\ref{theo:Gauss_Universality} )}\label{sec:optimalMMSE}
	
	In this section, we build on the universality results proven in Appendix~\ref{sec:Proof_Universality} and show that the limit of the free energy and overlaps have a variational representation. The power of the Gaussian universality (Lemma~\ref{lem:2nd_approx}) is that the limits of non-linear models have essentially been reduced to a linear spiked model with additive noise. These limiting varaitional formulas follow from existing variational formulas for the limits of finite rank matrix estimation problems after a change of prior. 
	
	\subsection{Limit of the Free Energy}
	We first define the \emph{replica symmetric functional} 
	\begin{align}
		\mathcal{F}(\Delta_{k_F},q) &= - \frac{q^2}{4\Delta_{k_F}} +  \E\log \bigg[ \int \exp \bigg(  \frac{\sqrt{q}}{\sqrt{\Delta_{k_F}}} z m^{k_F} +   \frac{m^{k_F} x^{k_F}}{\Delta_{k_F}}  - \frac{m^{2 {k_F}}}{2 \Delta_{k_F}}  \bigg) \,  \dprior(x) \bigg] \\
		&= 
		- \frac{q^2}{4\Delta_{k_F}} +
		\E\log \bigg[ \int \exp \bigg(  \frac{\sqrt{q}}{\sqrt{\Delta_{k_F}}} z m +   \frac{m \Tilde{x}}{\Delta_{k_F}}  - \frac{ m^2 }{2 \Delta_{k_F}}  \bigg) \,  \dpriork(m) \bigg]
		. \label{eq:RSformula_2}
	\end{align}
	where $\tilde {x} \sim x^\kf$. The following limit follows from a direct application of \cite[Theorem~13]{lelarge2017fundamental}.
	\begin{theo}[Limiting Free Energy]\label{theo:limitFE} 
		\begin{equation*}
			\lim_{N \to \infty}  F^{(G)}_N \left( \Delta_{k_F} , \pi_{k_F} \right) = \sup_{q} \mathcal{F}(q).     
		\end{equation*}
	\end{theo}
	\begin{proof}
		The proof is identical to the rank 1 matrix factorization results. This is proven rigorously by \cite{lelarge2017fundamental} and follows as a special case of \cite{guionnet2023estimating}. 
		
		It is worth remarking that the crucial concentration of the overlaps  $\E \langle R_{12}^{\odot \kf} \rangle = \E \langle \frac{\langle \vect{m}^\kf, \vect{x}^\kf \rangle}{N} \rangle$ follows from the fact that the Gaussian perturbation channel exists by the Schur product theorem. 
	\end{proof}
	
	\begin{remark}
		An explicit characterization of the limit of the constrained free energy $F_N(A,B;g,\pi)$ and the associated almost sure large deviations principle follows from a direct application of \cite[Theorem~2.6]{guionnet2023estimating} with prior $\pi_{\kf}$ after Gaussian universality. An explicit formula is not needed in this work, since we are in the Bayesian optimal setting and a simplification of the full Replica symmetry breaking formula is possible in this setting. It is worth noting that the almost sure large deviations principle shows that the overlaps are almost surely universal in the high dimensional limit, instead of only on average.
	\end{remark}
	
	In these models, the quantity $q$ represents the expected value of the overlap $R^{\odot \kf}_{12}$ under the Gibbs measure $\langle \cdot \rangle$ (up to a small Gaussian perturbation channel). 
	Using a special case of the rank~$1$ critical point condition in \cite[Equation~2.14]{AJFL_inhomo}, we have the following fixed point equation for the solution to the replica symmetric functional.
	
	\begin{cor}[Critical Point Condition for the Overlaps]\label{cor:fixedpoint}
		If $k_F$ is odd and $\prior$ is centered then
		the maximizers of $\phi$ defined in \eqref{eq:RSformula_2} satisfy the following fixed point equation
		\begin{equation}\label{eq:critpiontcond}
			q = \frac{1}{\Delta}\E \langle x^{k_F}  \rangle_q \langle x^{k_F} \rangle_q
		\end{equation}
		where $\langle \cdot \rangle_{q}$ denotes the average
		\[
		\langle f(x) \rangle_{q} = \frac{ \int f(x) \exp \bigg(  \frac{\sqrt{q}}{\sqrt{\Delta}} z x^{k_F} +   \frac{q x^k_F x^{k_F}_0}{\Delta}  - \frac{q x^{2k_F}}{2 \Delta}  \bigg) \, \dprior(x)} { \int \exp \bigg(  \frac{\sqrt{q}}{\sqrt{\Delta}} z x^{k_F} +   \frac{q x^{k_F} x^{k_F}_0}{\Delta}  - \frac{q x^{2\kf}}{2 \Delta}  \bigg) \, \dprior(x)}.
		\]
	\end{cor}
	\begin{proof}
		The critical point condition \eqref{eq:critpiontcond}, is simply the first order condition for the maximizer. This is found by differentiating the variational formula with respect to the paramter $q$. 
	\end{proof}
	The maximizer of the free energy recovers the limiting value of the overlaps similarly to the IMMSE theorem \cite{immse} and a direct application of the envelope theroem \cite{envelope}.
	\begin{cor}
		If $q^*$ maximizes $\mathcal{F}(\Delta_{k_F},q)$, then
		\[
		\lim_{N \to \infty} \E \langle R_{12}^{\odot \kf} \rangle =\lim_{N \to \infty}  \E \langle R_{12}^{\odot \kf} \rangle_G = q^*.
		\]
	\end{cor}
	\begin{proof}
		By universality of the overlaps \eqref{eq:universality_overlaps}, we have
		\[
		\lim_{N \to \infty} \E \langle R_{12}^{\odot \kf} \rangle =\lim_{N \to \infty}  \E \langle R_{12}^{\odot \kf} \rangle_G
		\]
		so it suffices to compute the limit of the latter. Under the change of variables $\frac{1}{2\Delta} \to \lambda$, we see that for every finite $N$ and an application of the Nishimori identity (see for example \cite[Proposition~16]{lelarge2017fundamental})
		\[
		\frac{d}{d\lambda}  F^{(G)}_N \left( \Delta_{k_F} , \pi_{k_F} \right) = \E \langle R_{12}^{\odot \kf} \rangle_G.
		\]
		In the limit, for any solution to the fixed point equation \eqref{eq:critpiontcond} and an application of the Nishimori identity, one sees that
		\[
		\frac{d}{d\lambda} \mathcal{F}(\Delta_{k_F},q) = \frac{q^2}{4}.
		\]
		The envelope Theorem implies that the limit of the derivative is equal to the derivative of the limit, so
		\[
		\lim_{N \to \infty} \E \langle R_{12}^{\odot \kf} \rangle_G = \frac{q^2}{4}
		\]
		where $q$ is the maximizer of $\mathcal{F}(\Delta_{k_F},q)$.
	\end{proof}
	
	Furthermore, by studying the stability of the fixed points around $0$, one can recover the information theoretical recovery transition on the generalized Fisher coefficient matrix.
	
	\begin{cor}[IT Recovery Transition]\label{cor:recovery} Suppose that Hypotheses~\ref{hyp:bayes} holds and $k_F$ is odd and $\prior$ is centered. Then for
		\[
		\frac{1}{\Delta}  < \frac{1}{\E[ x^{2 k_F} ]^2}
		\]
		the MMSE is zero. 
	\end{cor}
	
	\begin{proof}
		This is a direct consequence from the \cite[Lemma~2.15]{AJFL_inhomo} with a constant variance profile matrix. It is worth mentioning that a sharp bound can be proven if the priors are Gaussian because the Gaussian integrals can be computed explicitly, and the cloed form variational problem can be solved by hand. 
	\end{proof}
	
	\newpage
	\section{Rigorous Statement and Proof of the Spiked Decomposition of the Fisher Matrix (Thm.~\ref{th:informal_rank1decompositionFisher})}
	\label{sec:Proof_Spiked_Fisher}
	An informal statement of this result was given in Theorem~\ref{th:informal_rank1decompositionFisher}. In this section, we provide the precise statement and its proof.
	\begin{theo}[Universal Spiked Decomposition of the Fisher matrix]
		\label{th:rank1decompositionFisher}
		We have with probability going to $1$ as $N \to \infty$,
		\begin{align}
			\frac{\rdmmat{S}_{k_F}}{\sqrt{N}}  \ed   \frac{\rdmmat{Z}}{\sqrt{N}} +\frac{1}{\Delta_{k_F}} \mat{W}_{\kf}  + \rdmmat{E}
		\end{align}
		where 
		\begin{itemize}
			\item  $ \rdmmat{Z} = (Z_{ij})_{ 1 \leq i,j \leq N}$ is a symmetric matrix whose elements are independent  (modulo the symmetry) and identically distributed as $ Z \ed \frac{g^{(k_F)}(Y_0,0)}{k_F!} $ with $Y_0 \sim \pout(\cdot | 0)$;
			\item $\mat{W}_{\kf} := \frac{\vect{x}^{k_F}}{\sqrt{N}} \left( \frac{\vect{x}^{k_F}}{\sqrt{N}}  \right)^{\top}$;
			\item $\rdmmat{E}$ is a symmetric matrix with $\normop{\rdmmat{E}} \leq O(N^{-\frac{1}{2\kf}})$.
		\end{itemize}
		and $\Delta_{k_F}$ is the $k_F$-th order Fisher coefficient given by Eq.~\eqref{eq:fisherscores}.  Furthermore, $\rdmmat{Z}$ is independent of $\mat{W}_\kf$. 
	\end{theo}

	We prove the spiked decomposition in two steps. First we do an entrywise decomposition of the matrix in the form of a spiked matrix, up to an error term that is vanishing in the dimension. 
	\begin{lem}\label{lem:rank1decomp_step1}
		Almost surely in $(\rdmmat{Y}, \rdmmat{W})$, we have that
		\begin{align}\label{eq:intermediate_spike_decom}
			\frac{\rdmmat{S}_{k_F}}{\sqrt{N}}  = \frac{\rdmmat{\tilde Z}}{\sqrt{N}} +\frac{1}{\Delta_{k_F}} \mat{W}_{\kf}  + \rdmmat{E}
		\end{align}
		where 
		\begin{itemize}
			\item  $ \rdmmat{\tilde Z} = (\tilde Z_{ij})_{ 1 \leq i,j \leq N}$ is a symmetric matrix whose elements are independent  (modulo the symmetry) and identically distributed and given by
			\[
			\tilde Z_{ij} = \frac{1}{\sqrt{\Delta_\kf}} \frac{ S_{ij,k_F} -\mu_{ij} }{\sigma_{ij}} 
			\]
			where the conditional mean and variance parameters are
			\[
			\mu_{ij} =  \E ( S_{ij,k_F} | W_{ij} ) \qquad \text{and} \qquad \sigma^2_{ij} = 	\E [ (S_{ij,k_F} - \E ( S_{ij,k_F} | W_{ij} ) )^2 | W_{ij} ].
			\]
			Furthermore, the randomness in $\tilde Z$ is conditionally independent of the the spike term in \eqref{eq:intermediate_spike_decom}.
			\item $\mat{W}_{\kf} := \frac{\vect{x}^{k_F}}{\sqrt{N}} \left( \frac{\vect{x}^{k_F}}{\sqrt{N}}  \right)^{\top}$;
			\item $\rdmmat{E}$ is a symmetric matrix with $\normop{\rdmmat{E}} \leq O(N^{-\frac{1}{2\kf}})$
		\end{itemize}
		and $\Delta_{k_F}$ is the $k_F$-th order Fisher coefficient given by Eq.~\eqref{eq:fisherscores}.  
	\end{lem}
	
	\begin{proof}
		We now fix $W_{ij}$ and proceed conditionally on $W_{ij}$. In particular, for fixed $W_{ij}$, the random variables $Y_{ij}$ are conditionally independent. We have
		\[
		S_{ij,k_F} = S_{ij,k_F} - \E ( S_{ij,k_F} | W_{ij} ) + \E ( S_{ij,k_F} | W_{ij} ).
		\]
		By the higher order Fisher inequalities Lemma~\ref{lem:ho_identity}, we see that
		\[
		\E ( S_{ij,k_F} | W_{ij} ) = \frac{1}{\Delta_{k_F}} W_{ij}^{k_F} + E_{ij}
		\]
		where $E_{ij}$ is of order $O(N^{-\frac{1}{2\kf}-\frac{1}{2}})$, which gives us the second term in our decomposition. Notice that this entrywise bound implies that $\normop{E} \leq O(N^{-\frac{1}{2\kf}+\frac{1}{2}})$ by using the fact that $\normop{\rdmmat{E}} \leq \| \rdmmat{E} \|_F$.
		
		To control the first term, notice that the entries of $S_{ij,k_F}$ are independent conditionally on $W_{ij}$ and  
		\[
		S_{ij,k_F} - \E ( S_{ij,k_F} | W_{ij} )
		\]
		is a centered Wigner matrix with constant variance given by
		\[
		\E [ (S_{ij,k_F} - \E ( S_{ij,k_F} | W_{ij} ) )^2 | W_{ij} ].
		\] 
		Unfortunately, this variance profile is not constant because it depends on the $W_{ij}$. However, from the higher order approximate Fisher inequalities, in Lemma~\ref{lem:ho_identity} we see that
		\begin{equation}\label{eq:sigmadeltaapprox}
			\E [ (S_{ij,k_F} - \E ( S_{ij,k_F} | W_{ij} ) )^2 | W_{ij} ] =  \frac{1}{\Delta_{k_F}} + O(N^{-1/2}).
		\end{equation}
		We now claim that the matrix with entries
		\[
		\sigma_{ij}^2 = \E [ (S_{ij,k_F} - \E ( S_{ij,k_F} | W_{ij} ) )^2 | W_{ij} ]
		\]
		can be written as a Wigner matrix with iid entries with variance $\frac{1}{\Delta_{k_F}}$ up to some additive error. Notice that the matrix
		\[
		\rdmmat{Z} = \bigg[ \frac{1}{\sigma_{ij}} \E [ (S_{ij,k_F} - \E ( S_{ij,k_F} | W_{ij} ) )^2 | W_{ij} ] \bigg]_{ij}
		\]
		has iid entries with variance $1$, so it is a Wigner matrix. Furthermore, we have
		\[
		\rdmmat{S}_{k_F} - \E ( \rdmmat{S}_{k_F} | \mat{W} ) = \bigg( \rdmmat{ \frac{1}{\Delta_{k_F}} } + \rdmmat{E}  \bigg)^{\odot \frac{1}{2}} \odot \rdmmat{Z}
		\]
		where $\rdmmat{E}$ is a matrix with entries of order $O(N^{-\frac{1}{2}})$. Now, using bounds on the operator norms of the Hadamard products of matrices, we have
		\[
		\bigg\|\bigg( \rdmmat{ \frac{1}{\Delta_{k_F}} } + \rdmmat{E}  \bigg)^{\odot \frac{1}{2}} \odot \rdmmat{Z} - \bigg( \rdmmat{ \frac{1}{\Delta_{k_F}}  \bigg)^{\odot \frac{1}{2}}  }  \odot \rdmmat{Z} \bigg\|_{\mathrm{op}} \leq C \| \rdmmat{E}^{\odot \frac{1}{2}} \odot  \rdmmat{Z} \|_{\mathrm{op}} \leq O(N^{-1/2}) \| \rdmmat{Z} \|_{\mathrm{op}}
		\]
		by a direct application of the Cauchy--Schwarz inequality and the Frobenius bound of the operator norm. Therefore, we may replace the random variance profile with the constant term, at the cost of a matrix whose operator norm of order  $O(N^{1/4})$. This bound is still controllable in the limit, since we will divide all terms by $\sqrt{N}$. This proves the decomposition in \eqref{eq:intermediate_spike_decom}.
	\end{proof}
	
	This entrywise decomposition is useful for the AMP, where we are doing actual denoising on the Fisher matrix, and need an almost sure equality instead of in distribution. However, when proving thresholds, it will be useful to go one step further and characterize the limiting distribution of the noise matrix in the spiked decomposition. In the next lemma, we show that although $\tilde Z$ in \eqref{eq:intermediate_spike_decom} depends on the realization of $W_{ij}$, it does so in a weak way.
	
	Let 
	$ \rdmmat{Z} = (Z_{ij})_{ 1 \leq i,j \leq N}$ be a symmetric matrix whose elements are independent  (modulo the symmetry) and identically distributed as $ Z \ed \frac{g^{(k_F)}(Y_0,0)}{k_F!} $ with $Y_0 \sim \pout(\cdot | 0)$ as was defined in Theorem~\ref{th:rank1decompositionFisher}. In particular, $\rdmmat{Z}$ does not depend on $\rdmmat{W}$. Notice that $\rdmmat{Z}$ and $\rdmmat{\tilde Z}$ are defined such that they share the same first and second moments. The next lemma implies that we can replace $\rdmmat{\tilde Z}$ in Lemma~\ref{lem:rank1decomp_step1} with $\rdmmat{Z}$.  We now show that their laws are close in distribution.
	
	For simplicity, we measure the distances between the laws of the entries encoded by smooth functions with bounded third derivative \cite[Section~4]{chen2010normal}. Recall the $\cH$ norm is an integrable probability metric \cite{integrableprobmetric} between random variables $X$ and $Y$ with laws $\mathcal{L}(X)$ and $\mathcal{L}(Y)$ respectively
	\begin{equation}\label{eq:Hnorm}
		\| \mathcal{L}(X) - \mathcal{L}(Y) \|_{\cH} = \sup_{h \in \cH} | \E h(X) - \E h(Y)|
	\end{equation}
	We will pick the class of functions $\cH = L_{2,c}^\infty(\R)$, which denotes all smooth functions such that the function and its first and second derivatives are uniformly bounded by some constant $c$. We will prove that the laws of the entries of $\frac{\rdmmat{Z}}{\sqrt{N}}$ and $\frac{\rdmmat{\tilde Z}}{\sqrt{N}}$ are bounded by $O(cN^{-1-\frac{1}{2\kf}})$ through an application of the approximate Stein's Lemma (see for example \cite[Lemma~3.7]{PBook}). 
	\begin{lem}[Generalized Stein's lemma]\label{lem:gen_steinlemma}
		Let $f \in C^2(\R)$ be bounded with bounded first and second derivatives. Let X be a real
		centered random variable with finite third moment and $X'$ and independent copy of $X$. Then there is a constant $C$ such
		that
		\begin{align}
			| \E [f(X)X] - \E [f^{(1)}(X)X'^2]  |
			&\leq
			C \| f^{(2)} \|_{\infty} \E |X|^3
			\, .
		\end{align}
	\end{lem}
	We now prove the convergence result.
	
	\begin{lem}
		For all $N \geq 1$ and all $i < j$, we have
		\[
		\bigg\| \mathcal{L}\bigg(\frac{\tilde Z_{ij}}{\sqrt{N}}\bigg) - \mathcal{L}\bigg(\frac{Z_{ij}}{\sqrt{N}}\bigg) \bigg\|_{L_{2,c}^\infty} \leq O(c N^{-1-\frac{1}{2\kf}}).
		\]
	\end{lem}
	
	\begin{proof}
		Without loss of generality, we may assume that $f(0) = 0$ for any $f \in L_2^\infty (\R)$ since we can simply add and subtract a constant without changing the difference on the right hand side of \eqref{eq:Hnorm}. By a direct application of the higher order Fisher identities. 
		\begin{align*}
			&\E_Y f\bigg(\frac{\tilde Z_{ij}}{\sqrt{N}}\bigg) 
			\\&= \int f\bigg(\frac{1}{\sqrt{\Delta} \sqrt{N}} \frac{ \frac{g^{(k_F)}(y,0)}{k_F!} -\mu_{ij} }{\sigma_{ij}}\bigg) e^{g(y,W_{ij}) } \, dy
			\\&= \int f\bigg(\frac{1}{\sqrt{\Delta} \sqrt{N}} \frac{ \frac{g^{(k_F)}(y,0)}{k_F!} -\mu_{ij} }{\sigma_{ij}}\bigg) e^{g(y,0) } \bigg( 1 + \frac{g^{(k_F)}(y,0)}{k_F!} W_{ij}^\kf \bigg) \, dy + O\bigg(c N^{-\frac{\kf + 1}{2\kf} - \frac{1}{2} } \bigg)
			\\&= \E_{Y|0} f\bigg(\frac{\tilde Z_{ij}}{\sqrt{N}}\bigg)  + W_{ij}^\kf \E_{Y|0} Z f\bigg(\frac{\tilde Z_{ij}}{\sqrt{N}}\bigg)  +  O(c N^{-1 - \frac{1}{2\kf} } )
		\end{align*}
		where we used the fact that $f$ is $c$-Lipschitz since its first derivative is bounded, so $f(\tilde Z) = O(N^{-\frac{1}{2}})$ for every fixed $y$. Next, by the generalized Stein's lemma~\ref{lem:gen_steinlemma}, we see that
		\[
		\bigg| W_{ij}^\kf \E_{Y|0} Z f\bigg(\frac{\tilde Z_{ij}}{\sqrt{N}}\bigg) - \frac{W_{ij}^\kf}{\sqrt{N} \sqrt{\Delta_\kf \sigma_{ij} }} \E_{Y|0} f'\bigg(\frac{\tilde Z_{ij}}{\sqrt{N}}\bigg)  \bigg| \leq \frac{w |f''|}{N} = O(cN^{-\frac{3}{2}})
		\]
		since the second derivative term with respect to the $Z$ variable is of order $O( \frac{c}{N} )$ by the chain rule and $| W_{ij}^\kf| = O(N^{-\frac{1}{2}})$. We can conclude that
		\begin{equation}\label{eq:expansion_spiked_1}
			\E_Y f\bigg(\frac{\tilde Z_{ij}}{\sqrt{N}}\bigg) = \E_{Y|0} f( \tilde Z )  + \frac{W_{ij}^\kf}{\sqrt{N} \Delta_\kf \sqrt{\Delta_\kf \sigma_{ij} }} \E_{Y|0} f'(\tilde Z)  +  O( N^{-1 - \frac{1}{2\kf} } ).
		\end{equation}
		
		On the other hand, by Taylor's theorem with respect to $\mu_{ij} = \frac{1}{\Delta_\kf} W_{ij}^\kf + O(N^{-\frac{1}{2\kf} - \frac{1}{2}})$
		\begin{align}
			E_{Y|0} f\Big(\frac{Z_{ij}}{\sqrt{N}}\Big) &=  \E_{Y|0} f\Big(\frac{\sqrt{\Delta_\kf \sigma_{ij}} \tilde Z_{ij} + \mu_{ij}}{\sqrt{N}}\Big) 
			\\&= \E_{Y|0} f\Big(\frac{\sqrt{\Delta_\kf \sigma_{ij}} \tilde Z_{ij}}{\sqrt{N}}\Big) + \E_{Y|0} f'\Big(\frac{\sqrt{\Delta_\kf \sigma_{ij}} \tilde Z_{ij}}{\sqrt{N}}\Big) \frac{1}{\sqrt{N} \Delta_\kf} W_{ij}^\kf + O(cN^{\frac{3}{2}}) \label{eq:expansion_spiked_2}.
		\end{align}
		Furthermore, by \eqref{eq:sigmadeltaapprox} we see that
		\[
		\sqrt{\Delta_{\kf} \sigma_{ij}} = 1 + O(N^{-\frac{1}{4}})
		\]
		so this multiplicative factor introduces errors that are of lower order than the rest, thus \eqref{eq:expansion_spiked_1} and \eqref{eq:expansion_spiked_2} imply 
		\[
		\E_{Y|0} f\Big(\frac{Z_{ij}}{\sqrt{N}}\Big) = \E_{Y|0} f\Big(\frac{ \tilde Z_{ij}}{\sqrt{N}}\Big) + \E_{Y|0} f'\Big(\frac{\tilde Z_{ij}}{\sqrt{N}}\Big) \frac{1}{\sqrt{N} \Delta_\kf} W_{ij}^\kf + O(cN^{\frac{3}{2}})
		\]
		and
		\[
		\E_Y f\bigg(\frac{\tilde Z_{ij}}{\sqrt{N}}\bigg) = \E_{Y|0} f\Big(\frac{Z_{ij}}{\sqrt{N}}\Big)  + \frac{W_{ij}^\kf}{\sqrt{N} \Delta_\kf } \E_{Y|0} f'(\tilde Z)  +  O(c N^{-1 - \frac{1}{2\kf} } ).
		\]
		Taking the difference completes the proof.
	\end{proof}
	We now explain how this implies that we can replace $\tilde Z$ in \eqref{eq:intermediate_spike_decom}. Notice that the entries of $\frac{\rdmmat{Z}}{\sqrt{N}}$ and $\frac{\rdmmat{\tilde Z}}{\sqrt{N}}$ are close up to a random variable of order $O(cN^{-1-\frac{1}{2\kf}})$. By the Frobenius bound of the operator norm, this will imply that $\frac{\rdmmat{Z}}{\sqrt{N}}$ and $\frac{\rdmmat{\tilde Z}}{\sqrt{N}}$ are close up to some possible random matrix $\rdmmat{E}$, whose operator norm is of order $O(N^{-\frac{1}{2\kf}})$, which vanishes in the high dimensional limit. Furthermore, this implies that with  $\tilde Z_{ij}$ and $\tilde Z_{ij}$ have the same marginals asymptotically, so there exists a coupling such that $\tilde Z_{ij} = Z_{ij}$ with high probability. $\tilde Z_{ij}$ is independent of $W_{ij}$, so we can strengthen the conditional independence of $W_{ij}$ to independent in the high dimensional limit in Lemma~\ref{lem:rank1decomp_step1}. 
	\newpage
	\section{Rigorous Statement and Proof of the Top Eigenvector Decomposition of the Fisher Matrix (Thm.~\ref{theo:informal_eigvector_decomp_Fisher})}
	\label{sec:Proof_eigvect_decomposition}
	
	\begin{theo}[Universal Top Eigenvector Decomposition of the Fisher Matrix]\label{theo:eigvector_decomp_Fisher}
		Under the same hypothesis of Thm.~\ref{th:informal_rank1decompositionFisher}, let $\rdmvect{v}_1$ be the eigenvector of the Fisher matrix $\rdmmat{S}_{\kf}$ associated to its highest eigenvalue, with sign chosen such that $\langle \rdmmat{v}_1, \vect{y}/\sqrt{N} \rangle \geq 0$ without loss of generality and assume $\Delta_\kf < m_{2 \kf}^2$.
		
		Let  $J \subset \N$ be any subset of fixed size, then  conditioned on $\vect{x}$, as $N \to \infty$ we have:  
		\begin{align}
			\left( \sqrt{N} v_{1,i} \right)_{i \in J} \overset{\mathcal{D}}{\Rightarrow}
			\mathrm{q}_0 \, \bigg( \frac{x_i^{\kf}}{\sqrt{m_{2 \kf}} } \bigg)_{i \in J } + \sqrt{1- \mathrm{q}_0 ^2}   \, (g_i)_{_{i \in J} } ,
		\end{align}
		where $(g_i)_{i \in J}$ are independent standard Gaussian vectors and $\mathrm{q}_0 \equiv \mathrm{q}_0 (\Delta_\kf)$ is given by Eq.~\eqref{eq:def_limoverlap_PCA_Fisher}. 
	\end{theo}

	We first give a variant of Thm.~\ref{theo:informal_eigvector_decomp_Fisher} in a slightly more generic framework
	\begin{theo}[Top eigenvector decomposition for spiked Wigner matrices]\label{theo:eigvector_decomp_Spiked_Wigner}
		\indent Let $\rdmmat{X} := \rdmmat{W} + (\gamma/N) \vect{y} \vect{y}^{\top} + \rdmmat{E}$  where  $\rdmmat{W}$  is a Wigner matrix, $\vect{y} = (y_1,\dots,y_N) \sim \priortensor$ with variance one and $\gamma >1$ and $\normop{\rdmmat{E}} = O(N^\frac{1}{2\kf})$.  Let $\rdmvect{v}_1= (v_{1,1},\dots, v_{1,N})$  be the top eigenvector of $ \rdmmat{X}$ with sign chosen such that $\langle \rdmmat{v}_1, \vect{y}/\sqrt{N} \rangle \geq 0$ without loss of generality.
		
		Let  $J \subset \N$ be any subset of fixed size, then  conditioned on $\vect{y}$, as $N \to \infty$ we have  
		\begin{align}
			\left( \sqrt{N} v_{1,i} \right)_{i \in J} \overset{\mathcal{D}}{\Rightarrow}
			\,  \sqrt{1 - \gamma^{-2}} (y_i)_{i \in J } + \gamma^{-1}   \, (g_i)_{_{i \in J} } ,
		\end{align}
		where   $(g_i)_{i \in J}$ are independent standard Gaussian vectors. 
	\end{theo}

	It is immediate from Thm.~\ref{sec:Spike_decomposition_Fisher} and re-scaling of the parameters that Thm.~\ref{theo:eigvector_decomp_Spiked_Wigner} implies Thm.~\ref{theo:eigvector_decomp_Fisher}, thus in the following we prove and take the same notations as in Thm.~\ref{theo:eigvector_decomp_Spiked_Wigner}. 
	
	This decomposition and convergence of the eigenvector in the Wasserstein~2 distance is a direct consequence of \cite[Appendix~C]{Montanari2017EstimationOL}. To complement this derivation through the state evolution equations of an AMP, we sketch a proof of the equivalent fact using techniques from random matrix theory.
	
	\begin{proof}(Sketch)
		Following \cite{benaych-georges_eigenvalues_2011}, the eigenvector $\rdmvect{v}_1$ associated to the top eigenvalue $\lambda_1$ of the matrix $ \rdmmat{W} + \rdmmat{E} + (\gamma/N) \vect{y} \vect{y}^{\top}$ satisfies by definition the equation
		\begin{align}
			\label{eq:equation_eigvector}
			\left(  \rdmmat{W} + \rdmmat{E} + \frac{\gamma}{N} \vect{y} \vect{y}^{\top} \right) \rdmvect{v}_1 = \lambda_1 \rdmvect{v}_1,
		\end{align}
		which we can re-write as:
		\begin{align}
			\label{eq:equation_eigvector.2}
			\left( \lambda_1 - (\rdmmat{W} + \rdmmat{E})  \right) \rdmvect{v}_1 = \frac{\gamma}{N}  \langle \rdmvect{v}_1 , \vect{y} \rangle  \vect{y}.
		\end{align}
		Next, for a symmetric matrix $\mat{A}$ and $z \in \C \setminus \mathrm{Spec}(\mat{A})$, we define the \emph{Resolvent matrix} of $\mat{A}$  as:
		\begin{align}
			\mat{G}_{\mat{A}}(z) 
			&:=
			\big( z \mat{I} - \mat{A} \big)^{-1} \, .
		\end{align}
		From Eq.~\eqref{eq:equation_eigvector.2}, one has $\rdmvect{v_1} \propto \mat{G}_{\rdmmat{W} + \rdmmat{E}}(\lambda_1)   \vect{y}$, where the constant of proportionality is determined by imposing $\| \rdmvect{v}_1 \| =1$ which gives, after multiplication by $\sqrt{N}$,
		\begin{align}\label{eq:eigenvector_resolvent_relation}
			(\sqrt{N} \rdmvect{v}_1)
			&=
			\bigg \langle \frac{\vect{y}}{\sqrt{N}}  , \mat{G}_{\rdmmat{W} + \rdmmat{E}}(\lambda_1)^2  \frac{\vect{y}}{\sqrt{N}}  \bigg \rangle  ^{-1/2} \, \mat{G}_{\rdmmat{W}+ \rdmmat{E}}(\lambda_1) \vect{y}
		\end{align}
		\indent First, we argue that we can discard the contribution of the small matrix $\rdmmat{E}$ in this expression.
		
		\begin{lem}[Eigenvector Perturbation]\label{lem:error_matrix}
			\indent Under the same assumptions and notations of Thm.~\ref{theo:eigvector_decomp_Spiked_Wigner}, if we denote by $\rdmvect{v}'_1= (v'_{1,1},\dots, v'_{1,N})$ the top eigenvector of $ \rdmmat{W} + (\gamma/N) \vect{y} \vect{y}^{\top}$, with sign chosen such that $\langle \rdmmat{v}'_1, \vect{y}/\sqrt{N} \rangle \geq 0$, we have $ \| \sqrt{N} (v'_{i} - v_i)_{i \in J} \| \xrightarrow[N \to \infty]{\as} 0 $. 
		\end{lem}

		\begin{proof}
			Indeed, for $\gamma>1$, we know that for $N$ large enough, $\lambda_1 = \gamma + \gamma^{-1} +o(1)$ and separates from the spectrum of $\rdmmat{W}$ (and hence also from   $\rdmmat{W} + \rdmmat{E}$ since $\normop{\rdmmat{E}} \to 0$). Using the resolvent identity:
			\begin{align}
				\mat{G}_{\rdmmat{W} + \rdmmat{E}}( \lambda_1 )
				=
				\mat{G}_{\rdmmat{W}}( \lambda_1 ) +
				\mat{G}_{\rdmmat{W}}(\lambda_1) \, \rdmmat{E} \,  \mat{G}_{\rdmmat{W}+ \rdmmat{E}}( \lambda_1)
			\end{align}
			we get that for $N$ sufficiently large so that $ \normop{\rdmmat{E}} \leq \frac{|\lambda_1 - 2 |}{2}$ by the bound on the resolvent matrix and Weyl's inequality,
			\begin{align}
				\normop{ \mat{G}_{\rdmmat{W} + \rdmmat{E}}( \lambda_1 )
					-
					\mat{G}_{\rdmmat{W}}( \lambda_1 ) } \leq \bigg( \frac{2}{\|\lambda_1 - 2 \| } \bigg)^2
				\normop{ \rdmmat{E} }.
			\end{align}
			This operator norm bound implies that
			\begin{align}
				\bigg| \bigg\langle \frac{\vect{y}}{\sqrt{N}}  , \mat{G}_{\rdmmat{W} + \rdmmat{E}}(\lambda_1)^2  \frac{\vect{y}}{\sqrt{N}}  \bigg \rangle - \bigg \langle \frac{\vect{y}}{\sqrt{N}}  , \mat{G}_{\rdmmat{W}}(\lambda_1)^2  \frac{\vect{y}}{\sqrt{N}}  \bigg \rangle \bigg| \leq  \bigg( \frac{2}{\|\lambda_1 - 2 \| } \bigg)^2  \normop{ \rdmmat{E} } = C(\lambda_1) \normop{ \rdmmat{E} } \, .
			\end{align}
			and uniformly over $J \subset \N$ of fixed size,
			\begin{align}
				\| ( \mat{G}_{\rdmmat{W}+ \rdmmat{E}}(\lambda_1) \vect{y} - \mat{G}_{\rdmmat{W}}(\lambda_1) \vect{y} )_{i \in J} \|_2 \leq C^2 \bigg( \frac{2}{\|\lambda_1 - 2 \| } \bigg)^2
				\normop{ \rdmmat{E} } |J| = C(\lambda_1, |J|) \normop{ \rdmmat{E} } .
			\end{align}
			If we define $f(\rdmmat{A}) = \langle \frac{\vect{y}}{\sqrt{N}}  , \mat{G}_{\rdmmat{A}}(\lambda_1)^2  \frac{\vect{y}}{\sqrt{N}}   \rangle$ then \eqref{eq:eigenvector_resolvent_relation} implies that 
			\begin{align*}
				\| (\sqrt{N} \rdmvect{v}_1  - \sqrt{N} \rdmvect{v'}_1 )_{i \in J} \| &\leq | f(\rdmmat{W} + \rdmmat{E})^{-1/2} - f(\rdmmat{W} )^{-1/2} | \| \mat{G}_{\rdmmat{W}+ \rdmmat{E}}(\lambda_1) \vect{y} \|_2
				\\&\quad+ | f(\rdmmat{W} )^{-1/2} | \| (\mat{G}_{\rdmmat{W}+ \rdmmat{E}}(\lambda_1) \vect{y} -  \mat{G}_{\rdmmat{W}}(\lambda_1) \vect{y} )_{i \in J}\|_2
			\end{align*}
			all of which can be controlled using our operator norm bounds and the fact that 
			\begin{align*}
				\max(f(\rdmmat{W}),  f(\rdmmat{W} + \rdmmat{E}) ) > 0
			\end{align*}
			uniformly over $\vect{y}$ and $N$ sufficiently large so that $ 	\normop{ \rdmmat{E} }\leq \frac{\lambda_1 -2}{2}$. We conclude that for $J \subset \N$ of fixed size, and any $i \in J $, 
			\begin{align}
				\| \sqrt{N} (v'_{i} - v_i)_{i \in J} \|  \leq  C(\lambda_1, |J|) 
				\normop{ \rdmmat{E} } ,
			\end{align}
			where the constant $C(\lambda_1, |J|) $ is independent of the dimension $N$. 
		\end{proof}
		\indent  Next,  from \cite[Proposition~2.1]{PizzoRenfrew}, the quadratic form of $\vect{y}/\sqrt{N}$ with respectively $\mat{G}_{\rdmmat{W}}$ and $\mat{G}_{\rdmmat{W}}^2$ concentrate and are given by;
		\begin{align}
			\bigg \langle \frac{\vect{y}}{\sqrt{N}}  , \mat{G}_{\rdmmat{W}}(\lambda_1)  \frac{\vect{y}}{\sqrt{N}}  \bigg \rangle
			&=
			\frac{1}{\gamma} +o(1)
			\\
			\bigg \langle \frac{\vect{y}}{\sqrt{N}}  , \mat{G}_{\rdmmat{W}}(\lambda_1)^2  \frac{\vect{y}}{\sqrt{N}}  \bigg \rangle
			&=
			\frac{1}{\gamma^2 -1} + o(1)
		\end{align}
		
		As a consequence, adding and subtracting $y_j/\gamma$ to the $i$-th component of $\sqrt{N} \rdmvect{v}_1$ and denoting by $G_{ij}(\lambda_1) \equiv \big (\mat{G}_{\rdmmat{W}}(\lambda_1) \big)_{ij}$, we have:
		\begin{align}
			(\sqrt{N} {v}_{1,i})
			&= 
			\sqrt{1-\gamma^{-2}} \, y_i
			+
			\frac{ \sqrt{1-\gamma^{-2}}}{\sqrt{N}} \sum_{j=1}^{N}  \sqrt{N} \left(  G_{ij}(\lambda_1) - \gamma^{-1} \delta_{ij}  \right) y_j 
			+o(1) .
		\end{align}
		
		It remains to show that 
		\[
		\frac{ \sqrt{1-\gamma^{-2}}}{\sqrt{N}} \sum_{j=1}^{N}  \sqrt{N} \left(  G_{ij}(\lambda_1) - \gamma^{-1} \delta_{ij}  \right) y_j  \overset{\mathcal{D}}{\Rightarrow}  \gamma^{-1}   \, (g_i)_{_{i \in J} }  . 
		\]
		This will follow from the from the central limit theorem for eigenvalues in deformed matrices. Consider a finite collection of $u_1, \dots, u_l$ of non-random vectors in $l^2(\N)$, we have and let $u_p^{(N)}$ denote the projection of $u_p$ to the subspace spanned by the first $N$ standard basis vectors $e_1, \dots, e_N$. By \cite[Theorem~1.7]{PizzoRenfrew}, 
		\begin{align}
			\sqrt{N}\Big\langle u_p^{(N)} \left( G_{ij}(\lambda_1) - \gamma^{-1} \delta_{ij} \right), u^{(N)}_q \Big\rangle \overset{\mathcal{D}}{\Rightarrow} \gamma^{-2} \Big\langle u_p (\Tilde{W}_{ij} + G_{ij}), u_q \Big\rangle
		\end{align}
		where $\Tilde{W}_{ij} = (\sqrt{N} \rdmmat{W})_{ij}$ and the $G_{ij}$ are independent (modulo the symmetry $G_{ij} = G_{ji}$) Gaussian random variable with variance given by:
		\begin{align}
			\Var (G_{ii}) &= \frac{m_4 -3}{\gamma^2} +  \frac{2}{\gamma^2 -1}, \\
			\Var (G_{ij}) &= \frac{1}{\gamma^2 -1} \quad \mbox{for } i \neq j \, .
		\end{align}
		To see how this result can be applied to our setting, notice that 
		\[
		\sum_{j=1}^{N}  \sqrt{N} \left(  G_{ij}(\lambda_1) - \gamma^{-1} \delta_{ij}  \right) \frac{y_j}{\sqrt{N}} =  \sqrt{N} \langle e_i , (G(\lambda) - \gamma^{-1} I) \tilde y \rangle
		\]
		where $e_i$ is the standard $i$th basis vector and $\tilde y = \frac{y}{\sqrt{N}}$. Since the matrix $G_{ij}$ has independent entries for fixed $i$, we can apply the CLT again, which completes the proof of this result. 
		
	\end{proof}
	
	\newpage
	\section{Rigorous Statement and Proof of Optimal Detection and Performance for PCA on Non-Linear Channels (Thm.~\ref{theo:informal_PCA_optimality}) }\label{sec:proof_optimal_PCA}
	Let us consider the  model
	\begin{align}
		\label{eq:nonlinear_channel}
		Y_{ij} = h \bigg(Z_{ij} + \frac{\gamma(N)}{\sqrt{N}} x_i x_j \bigg)
	\end{align}
	where $Z \sim p_Z(.)$ is such that $p_Z(.) \in C^{k_0}$, for some $k_0\geq 1$ sufficiently large and with $p_Z >0$ almost everywhere.
	\jump We assume $h$ to be locally $C^1$ such that there exists a (possibly infinite) union of intervals $I_n$ on which $h$ is monotonic and we denote by $h_n^{\langle -1 \rangle }(.)$ the right-inverse for the composition of the function $h(.)$ on $I_n$ and  $h_n^{\langle -1 \rangle }(.)$ is assumed to be differentiable on $I_n$.
	\jump We first express the function entering the definition of the Fisher matrix $\rdmmat{S}_\kf$ in terms of the non-linearity $h$ and the density $p_Z$.
	\begin{lem}
		\label{lem:Fisher_function_non_linear}
		For $Y_{ij}$ given by Eq.~\eqref{eq:nonlinear_channel} with $\gamma_N(N) = N^{\beta}$, we have for $k \in \{1, \dots, \kf \}$: 
		\begin{align}
			g^{(k)}(y,0)= \frac{(\partial^{k}_w \pout)(y|0) }{\pout(y|0)}  :=  \sum_n (-1)^{k} \frac{p_Z^{(k)} \big( h_n^{\langle -1 \rangle }(y) \big )}{p_{Y_0}(y)} \Big |  (h_n^{\langle -1 \rangle})'(y) \Big | \, ,
		\end{align}
		where $Y_0 \ed h(Z)$ and $p_{Y_0} \equiv \pout(y,0)$ is the associated distribution.
	\end{lem}
	\begin{proof} The first equality is from Eq.~\eqref{eq:gk_to_partial_pout_k_leq_kf} since $k \leq \kf$ and the second one is a simple consequence of differentiating $k$ times the change of variable formula:
		\begin{align}
			\pout(y,w) = \sum_{n} p_Z \big(  h_n^{\langle -1 \rangle}(y) -w \big)  \bigg |  (h_n^{\langle -1 \rangle})'(y) \bigg | \, .
		\end{align}
	\end{proof}
	Let us introduce the following set of functions:
	\begin{align}
		C^{\kf}_{0}(h,Z) :=  \left\{ f  \text{ such that } \E f \big( h(Z) \big) =0   \text{ and }  \bigg | \int p_Z^{(k)}(z) f \big(h(z)) \dd z \bigg| < \infty  \text{ for } 1 \leq k \leq \kf \right\} \, .
	\end{align}
	let us remark that whenever $f$ and $h$ are smooth enough, the second condition can be simply written as $ | \E (f \circ h)^{(k)} (Z) | < \infty$ for $ 1 \leq k \leq \kf$ by integration by part.  
	\jump
	For $f \in C^{\kf}_{0}(h,Z)$, we denote by $ \rdmmat{F}^{(f)} =( F^{(f)}_{ij} )_{1 \leq i,j \leq N}$, the matrix given by the (entrywise) transformation of $Y_{ij}$ by the function $f$: 
	\begin{align}
		\label{eq:entriwise_transformation_matF}
		F^{(f)}_{ij}  := f \big( Y_{ij} \big) = f\bigg( h \bigg(Z_{ij} + \frac{\gamma(N)}{\sqrt{N}} x_i x_j \bigg) \bigg) \, . 
	\end{align}
	Note that the condition $\E f(Z) = 0$ ensures that the matrix  $\rdmmat{F}^{(f)}$ does not have a Perron-Frobenius mode. If $\pi$  has a non-zero mean, one gets a spurious correlation between the Perron Frobenius mode and the signal for this reason. 
	Eventually, let us denote the relevant scaling of the SNR in Eq.~\eqref{eq:nonlinear_channel}  by:
	\begin{align}
		\gamma_{\kf}(N) := \gamma_0 N^{\frac{1}{2}(1-1/\kf)}  \, . 
	\end{align}
	Our main result in this section shows that amongst all functions $f \in  C^{\kf}_{0}(h,Z)$, the one that achieves the lowest threshold for the appearance of positive overlap with (a power of) the signal $\vect{x}$ is given by the function in Lem.~\ref{lem:Fisher_function_non_linear}, which corresponds to the Fisher matrix $\rdmmat{S}_\kf$. Precisely, we have: 
	
	\begin{theo}[Optimal Detection and Performance for PCA on Non-Linear Channels]\label{theo:PCA_optimality}
		Assume $Y_{ij}$ is given by Eq.~\eqref{eq:nonlinear_channel}, $\rdmmat{F}^{(f)}/\sqrt{N}$  is given by Eq.~\eqref{eq:entriwise_transformation_matF}, $\rdmmat{S}_\kf/\sqrt{N}$ is given  by applying $g^{(\kf)}(.,0)$  in Lem.~\ref{lem:Fisher_function_non_linear} to $\rdmmat{Y}$. We denote by  $\rdmvect{v}_1^{(f)}$ (resp. $\rdmvect{v}_1^{\star}$),  the top eigenvector of $\rdmmat{F}^{(f)}/\sqrt{N}$ (resp. $\rdmmat{S}_\kf/\sqrt{N}$), then for any $f \in C^{\kf}_{0}(h,Z)$ we have:
		\begin{itemize}
			\item for $\gamma(N)$ such that $\gamma(N)/\gamma_{\kf}(N) \to 0$, the matrix $\rdmmat{F}^{(f)}/\sqrt{N}$ is a Wigner matrix with high probability and in particular $\rdmvect{v}_1^{(f)}$ has a vanishing overlap with any positive power of the signal vector $\vect{x}$.
			\item for $\gamma(N) = \gamma_{\kf}(N)$, the matrix $\rdmmat{F}^{(f)}/\sqrt{N}$ is with high probability either a Wigner matrix or a spiked Wigner matrix with a spike (rank-one vector) proportional to $\vect{x}^{\kf}$.  If we denote by $\gamma_{0,\mathrm{cr}}^{(f)} \in \R_+ \cup \{+ \infty \}$ (resp. $\gamma_{0,\mathrm{cr}}^{\star}$)  the critical threshold of $|\gamma_0|$ after which one observes an outlier in the spectrum of  $\rdmmat{F}^{(f)}/\sqrt{N}$ (resp. of $\rdmmat{S}_{\kf}/\sqrt{N}$), we have:
			\begin{itemize}
				\item  (optimality in detection)
				\begin{align}
					\gamma_{0,\mathrm{cr}}^{\star}
					&\leq
					\gamma_{\mathrm{cr}}^{(f)} \, , 
				\end{align} 
				\item  (optimality in performance)
				\begin{align}
					\lim_{N \to \infty}\overline{\mathrm{MSE}}_{N} \bigg(\vect{x}^{\kf},  \sqrt{N m_{2 \kf}} \rdmvect{v}_1^{\star} \bigg) 
					&\leq
					\lim_{N \to \infty}\overline{\mathrm{MSE}}_{N} \bigg(\vect{x}^{\kf},  \sqrt{N m_{2 \kf}} \rdmvect{v}_1^{(f)} \bigg) 
					.
				\end{align}
			\end{itemize}
		\end{itemize}
		Furthermore, if $\pi$ has mean zero, one has the same optimality results if one removes the assumption $\E f(h(Z)) = 0$ in $C^{\kf}_{0}(h,Z)$. 
	\end{theo}

	\begin{proof}
		The idea of the proof is to compare the threshold in both cases. 
		\begin{itemize}
			\item On the one hand, for the Fisher matrix $\rdmmat{S}$, we have from Thm.~\ref{th:rank1decompositionFisher} 
			\begin{itemize}
				\item if $\gamma(N)/\gamma_{\kf}(N) \to 0$, then the Fisher matrix is up to a vanishing error term a Wigner matrix without any spike.
				\item if $\gamma(N)= \gamma_{\kf}(N)$, then the critical threshold is:
				\begin{align}
					(\gamma_{0,\mathrm{cr}}^{\star})^\kf =   \frac{1}{m_{2\kf}} \E_{Y_{ij}|0} \left[ \left( \frac{ g^{(\kf)}(Y_{ij},0)}{\kf!}\right)^2 \right]^{-1/2} \, .
				\end{align}
			\end{itemize}
			\item On the other hand, from \cite{guionnet2023spectral}, when applying a non-linearity $\tilde{f}$ to a spiked model:
			\begin{align}
				F^{(\tilde{f})}_{ij} := \tilde{f} ( Z_{ij} + \gamma(N)/\sqrt{N} x_i x_j) \, ,
			\end{align}
			then if we denote by $l_{\star} \equiv l_\star(\tilde{f}) := \inf_{ l \in \mathbb{N}_*} \{ \mathbb{E} \tilde{f}^{(l)}(Z) \neq 0  \}$, we have
			\begin{itemize}
				\item if $\gamma(N) / N^{1/2(1-1/l_\star)} \to 0$, the matrix with entries $F^{(\tilde{f})}_{ij}/\sqrt{N}$ is with high probability and up to a vanishing error term,  a Wigner matrix without any spike.
				\item  if $\gamma(N) = \gamma_0 N^{1/2(1-1/l_\star)} $, then the matrix $F^{(\tilde{f})}_{ij}/\sqrt{N}$ is a spiked Wigner and has a the threshold satisfying the bound
				\begin{align}
					\label{eq:bound_threshold_Spectral}
					|\gamma_{\mathrm{cr}}(\Tilde{f})|^{l_\star}  \geq \frac{1}{m_{2 l_\star}}\frac{|\E \Tilde{f}^{(l_\star)}(Z)|}{\sqrt{\E \tilde{f}(Z)^2}} \, ,
				\end{align}
				with equality if and only if $\Tilde{f} \in H^4(p_Z)$.
			\end{itemize}
			This case corresponds to our matrix $\rdmmat{F}^{(f)}$ by setting $\tilde{f} =(f \circ h)$. 
		\end{itemize}
		To show our result for the threshold, it is then enough to prove that 
		\begin{align}
			l_{\star}(f \circ h) \geq \kf
		\end{align}
		and the bound 
		\begin{align}
			|\E (f \circ h)^{(\kf)}(Z)|^2 \leq \E ((f \circ h)(Z))^2  \, \E_{Y_{ij}|0} \left[ \left( \frac{ g^{(\kf)}(Y_{ij},0)}{\kf!}\right)^2 \right] \, .
		\end{align}
		Furthermore, since the expression of the overlap in \cite{guionnet2023spectral} only depends on the ratio of the form of the RHS of Eq.~\eqref{eq:bound_threshold_Spectral}, this bounds is enough to get the inequality for the MSE. 
		\jump 
		
		\begin{itemize}
			\item  let us first fix $l=1$ and prove the bound
			\begin{align}
				\bigg | \E (f \circ h)'(Z) \bigg |^2 
				&\leq
				\E_{Y_0} f(Y_0)^2 \,  \E_{Y_0}  g^{(1)}(Y_0,0) \, .
			\end{align}
			If $\kf >1$, then the right-hand side (RHS) of this equation is zero by definition of $\kf$. This implies that the left-hand side is also equal to zero and thus $l_{\star}( f \circ h) >1$ by definition of $l_{\star}( f \circ h)$. If $\kf =1$, this bound gives the desired hierarchy for the threshold. 
			
			By integration by parts we have 
			\begin{align}
				\bigg | \E (f \circ h)'(Z) \bigg | = \bigg | \int_{\R} \dd z \, p_Z'(z) f(h(z))  \bigg | \, ,
			\end{align}
			and let us now cut the integral on the real line as a (possibly infinite) sum of integral on intervals $I_n$ such that on $I_n$, the function $h$ is monotonic:
			\begin{align}
				\bigg | \E (f \circ h)'(Z) \bigg | = \bigg | \sum_n \int_{I_n} \dd z \, p_Z'(z) f(h(z))  \bigg | \, ,
			\end{align}
			for each interval $I_n$, let us perform the change of variable $y = h(z) \leftrightarrow z =  h_n^{\langle -1 \rangle }(y) $  where we recall $h_n^{\langle -1 \rangle }(.)$ is the right-inverse for the composition of the function $h(.)$ on the interval $I_n$. To lighten the notations, let us also denote by 
			\begin{align}
				\Tilde{q}^{(1)}_n(y) \equiv p_Z' \big(  h_n^{\langle -1 \rangle }(y) \big) \1(y \in J_n) \, ,
			\end{align}
			and
			\begin{align}
				u_n(y) \equiv  |  (h_n^{\langle -1 \rangle })'(y) | \1(y \in J_n)   \, , 
			\end{align}
			where $J_n$ is the image of $I_n$ by $h_n^{\langle -1 \rangle}(.)$. We have:
			\begin{align}
				\bigg | \E (f \circ h)'(Z) \bigg | 
				&= 
				\bigg | \sum_n \int_{J_n}  \, \Tilde{q}^{(1)}_n(y)  u_n(y) f(y)  \dd y \bigg | = \bigg | \sum_n \int_{\R}  \, \Tilde{q}^{(1)}_n(y)  u_n(y) f(y)  \dd y \bigg | .
			\end{align}
			Applying Fubini's theorem implies 
			\begin{align}
				\bigg | \E (f \circ h)'(Z) \bigg | &\leq \int_{\R}  \bigg |  f(y) \, \sum_n  \Tilde{q}^{(1)}_n(y) u_n(y) \bigg |  \dd y
				\, .
			\end{align}
			
			Next let us multiply and divide the  integrand by $p_{Y_0}(y) >0$ and recalling the definition in Lem.~\ref{lem:Fisher_function_non_linear}, we have:
			\begin{align}
				\bigg | \E (f \circ h)'(Z) \bigg |  &\leq  \int_{\R}  \bigg |f(y)  \, \, (\partial_w g) (y,0) \bigg|  p_{Y_0}(y) \dd y \equiv \E_{Y_0} \Big| f(Y_0) (\partial_w g) (Y_0,0)  \Big| 
				\, ,
			\end{align}
			which gives the desired result by Cauchy-Schwartz inequality.
			\item Let us now consider $\kf \geq 2$ and $1 < l \leq \kf$. If we introduce $\Tilde{q}^{(l)}_n(y) \equiv p_Z^{(l)} \big(  h_n^{\langle -1 \rangle }(y) \big)  \1( y \in J_n)$, with $u_n(y)$ as before, the same application of cutting the integral into the parts where $h$ is monotonic and performing the same change of variable $y = h(z) \leftrightarrow z =  h_n^{\langle -1 \rangle }(y)$, gives, with Fubini's theorem,  the bound:
			\begin{align}
				\bigg | \E (f \circ h)^{(l)}(Z) \bigg | &\leq \int_{\R}  \bigg |  f(y) \, \sum_n  \Tilde{q}^{(l)}_n(y) u_n(y) \bigg |  \dd y
				\, .
			\end{align}
			Multiplying and dividing the integrand by $p_{Y_0} >0$ gives
			\begin{align}
				\bigg | \E (f \circ h)^{(l)}(Z) \bigg |  
				&\leq  
				\E \bigg | f(Y_0)  \frac{(\partial^{l}_w \pout)(y|0) }{\pout(y|0)} \bigg | \, , 
			\end{align}
			and thus from Lem.~\ref{lem:Fisher_function_non_linear} (since $k \leq \kf$); this is also also equal to 
			\begin{align}
				\bigg | \E (f \circ h)^{(l)}(Z) \bigg |  
				&\leq  
				\E \big | f(Y_0)  g^{(l)}(Y_0,0) \big | \\
				\bigg | \E (f \circ h)^{(l)}(Z) \bigg |  &\leq  \E  f(Y_0)^2  \,   \E g^{(l)}(Y_0,0)^2  \, .
				\label{eq:Cauchy_schwartz_non_linear_k}
			\end{align}
			\begin{itemize}
				\item For $l  < \kf$: the RHS of Eq.~\eqref{eq:Cauchy_schwartz_non_linear_k} is equal to zero and  so we have $l_\star ( f \circ h) \geq \kf$ as expected. 
				\item  for $l = \kf$, we get the desired bound and this concludes the proof.
			\end{itemize}
		\end{itemize}		
	\end{proof}
	\newpage
	\section{Proof of AMP converges to SE (Thm.~\ref{theo:AMP_performance})}
	\label{sec:Proof_AMP}
	
	\subsection{An AMP Algorithm for the Spiked Wigner Model}\label{sec:classicalAMP}
	
	We first state the classical state evolution theorems \cite{bayati2011dynamics,berthier2020state,gerbelot2023graph}. We begin by recalling the classical AMP iterations for the spiked Wigner model
	\begin{align}
		\label{eq:linear_channel}
		Y_{ij} = Z_{ij} + \frac{\gamma}{\sqrt{N}} x_i x_j.
	\end{align}
	Let $(g_t)_{t = 0}^\infty$ be a sequence of Lipschitz denoising functions on $\R$ with weak derivatives $g'_k$. We set $\vect{\hat v^{-1}} = 0$ and initializer $\vect{v^0}$. In this context, we are interested in the case of a spectral initialization~\cite{Mondelli_spectralamp2}. The AMP iterates for the spiked Wigner model are iterates of the form
	\begin{equation}
		\vect{\hat v^t} = g_t (\vect{v^t}), \qquad \vect{b_t} = \frac{1}{N} \sum_{i = 1}^N g'_t( v_i^t ), \qquad \vect{v^{k + 1}} = \rdmmat{Y} \vect{\hat v^t} - b_t \vect{\hat v^{t - 1}}
	\end{equation}
	
	The main goal is to characterize the high dimensional distribution of the iterates through pseudo Lipschitz test functions. We denote the space $\mathrm{PL}_D(r,L)$ the class of functions from $\psi: \R^D \mapsto \R$ such that
	\[
	|\psi(\vect{x}) - \psi(\vect{y})| \leq L \|\vect{x} - \vect{y}\| (1 + \|\vect{x} \|^{r-1} + \|\vect{y}\|^{r-1}).
	\]
	
	There are usually 3 conditions on these linear channels (see \cite[Section~3.1]{feng2022unifying}). 
	\begin{enumerate}
		\item[(M0)] The noise matrix $Z$ is Gaussian and independent of the initializer $\rdmvect{v}$ and signal $\vect{x}$. 
		\item[(M1)] There exist $\mu_0, \sigma_0 \in \R$ and independent random variables $U$, $V$ with $\E U^2 = \E V^2 = 1$, such that 
		\[
		\sup_{\psi \in \mathrm{PL}_2(2,1)} \bigg| \frac{1}{N} \sum_{i = 1}^N\psi(v_i^0, v_i) - \E \psi( \mu_0 V + \sigma_0 U, V) \bigg| \to 0. 
		\]
		\item[(M2)] The functions $g'_t$ are continuous Lebesgue almost everywhere for all $t$.
	\end{enumerate}
	
	It is shown in \cite{deshpande2014information} that the high dimensional limit of the joint law of the empirical distribution is characterized by the state evolution parameters
	\begin{equation}\label{eq:state}
		\mu_1 = \gamma \E ( V g_0(\mu_0 V + \sigma_0 G) ), \quad \mu_{t + 1} = \gamma \E ( V g_t(\mu_t V + \sigma_t G) )
	\end{equation}
	\begin{equation}
		\sigma_1 = \E g_0(\mu_0 V + \sigma_0 G)^2, \quad \sigma_{t + 1}^2 = \E g_t(\mu_t V + \sigma_t G)^2 
	\end{equation}
	and $V \sim \pi$ and $G \sim N(0,1)$.
	\begin{theo}[State Evolution for Spiked Models]
		Suppose that Hypothesis $(M0)$ to $(M2)$ hold and $Y$ is generated according to the linear channel \eqref{eq:linear_channel}. For each $k \in N$, we have
		\begin{equation}\label{eq:state_evo_char}
			\sup_{\psi \in  \mathrm{PL}_{t + 2}(2,1)} \bigg| \frac{1}{N} \sum_{i = 1}^N \psi(v_i^0, \dots, v_i^t, x_i) - \E( \psi( \mu_0 V + \sigma_0 U, \mu_1 V + \sigma_1 G_1 ,  \dots, \mu_t V + \sigma_t G_k, V ) )  \bigg| \to 0
		\end{equation}
		where $(\sigma_1 G_1, \dots, \sigma_{t - 1}G_{t-1}) \sim N(0, \Sigma)$ and independent of $U,V$ in Hypothesis $M(1)$ and
		\[
		\Sigma_{k,l} = \begin{cases}
			\E g_{0}(\mu_{0}V + \sigma_{0} U) g_{k-1}(\mu_{k-1}V + \sigma_{k-1} G)	& l = 1\\
			\E g_{l-1}(\mu_{l-1}V + \sigma_{l-1} G) g_{k-1}(\mu_{k-1}V + \sigma_{k-1} G)	& l \geq 2.
		\end{cases}
		\]
	\end{theo}
	This master theorem characterizes the high dimensional limit of the iterates. When one uses the Bayesian optimal denoising functions, one sees that overlaps of the inner products satisfy the same fixed point equations as the Bayesian optimal estimators. This implies that AMP produces IT optimal estimators. 
	\begin{cor}[Optimality of the AMP Algorithm for Spiked Models]\label{cor:optimality_spikedwigner}
		Suppose that Hypothesis $(M0)$ to $(M2)$ hold and $Y$ is generated according to the linear channel \eqref{eq:linear_channel}. Let the denoising functions $g_t = \eta$ be the Bayesian optimal ones defined in \eqref{eq:Gauss_denoise}. Then the limiting overlap $\lim_{t \to \infty} \langle \vect{v}^t, \vect{x} \rangle$ satisfies the following fixed point equation
		\begin{equation} \label{eqn:fixed point}
			\mu = \E_{x \sim \pi,G \sim \mathsf{N}(0,1)} \, \left[ \eta \left( \gamma \mu, \gamma \mu x + \sqrt{\gamma \mu} G \right) x \right].
		\end{equation}
		Furthermore, if $\mu_\star$ is the unique fixed point of \eqref{eqn:fixed point}, then
		\begin{align}
			\lim_{t \to \infty}\lim_{N \to \infty}\overline{\mathrm{MSE}}_{N}( \vect{v}^t , \rdmvect{x}) = 1 - \left(\frac{\mathrm{q}_{\star}}{m_{2}} \right)^2 
		\end{align}
		which matches the IT optimal MSE for spiked Wigner models \eqref{eq:critpiontcond}.
	\end{cor}
	\begin{proof} This statement is a direct application of Corollary~3.2 of \cite{feng2022unifying}. 
		The proof follows from the observation that if we take the Lipschitz function $\psi(v^0, \dots, v^t, x) = v^t x$ and apply \eqref{eq:state_evo_char}, we see that
		\[
		| \gamma \langle \vect{v}^t , \vect{x} \rangle - \gamma \E_{V\sim \pi} [ V g_t(\mu_t V + \sigma_t G) ]  | = | \gamma \langle \vect{v}^t , \vect{x} \rangle - \mu_{t - 1} | \to 0.
		\]
		For the Bayes optimal denoisers $\eta$, the Nishimori identity implies that $\mu_t = \sigma^2_t$. Furthermore, to remain consistent with the notation earlier in the paper, we can define $q = \frac{\mu}{\gamma}$. Therefore, by \eqref{eq:state} the $\mu^t$ in the limit satisfy the fixed point equation
		\[
		q = \E_{V \sim \pi} ( V g_t( \gamma q V + \sqrt{\gamma q} G) ),
		\]
		which matches the IT optimal overlaps \eqref{eq:critpiontcond}. In the bayesian optimal case, it is clear that the limiting MSE is a given by
		\[
		\lim_{t \to \infty} \lim_{N \to \infty}\overline{\mathrm{MSE}}_{N}( \vect{v}^t , \rdmvect{x}) = 1 - \left(\frac{\mathrm{q}_{\star}}{m_{2}} \right)^2 
		\]
		where the $m_2$ comes from the normalization factor requried to make $\vect{x}$ have entries of variance $1$. 
	\end{proof}
	
	\subsection{An AMP Algorithm for the Fisher Matrix}\label{sec:AMP_Fisher}

	Our starting point is the spectral decomposition of the Fisher matrix in Theorem~\ref{th:rank1decompositionFisher}
	\begin{align}
		\frac{\rdmmat{S}_{k_F}}{\sqrt{N}}  \ed   \frac{\rdmmat{Z}}{\sqrt{N}} + \frac{1}{\Delta_{k_F}} \mat{W}_{\kf}  + \rdmmat{E}
	\end{align}
	
	Due to the approximate spiked structure of the Fisher matrix, we claim that the \eqref{sec:classicalAMP} can be extended to this setting. We address the critical assumptions
	\begin{enumerate}
		\item[(M0)] Through the universality of AMP algorithms \cite{chen_lam_universal}, the Gaussian assumption on the noise matrix was generalized to symmetric matrices with unit variance and $\sigma$-sub Gaussian entries. The Fisher noise matrix $\rdmmat{Z}$ is a symmetric matrix whose elements are independent  (modulo the symmetry) and identically distributed as $ Z \ed \frac{g^{(k_F)}(Y_0,0)}{k_F!} $ with $Y_0 \sim \pout(\cdot | 0)$, so it falls under the same Gaussian universality class for AMP algorithms.
		\item[(M1)] Samples from independent factorized priors clearly satisfy the convergence of the empirical distributions by the law of large numbers. 
		\item[(M2)] We are using the Gaussian denoising functions $\eta$, clearly have derivatives that are continuous Lebesgue almost everywhere.
	\end{enumerate}
	
	All that remains is to show that the addition of a matrix $\rdmmat{E}$ with vanishing operator norm $\normop{\rdmmat{E}}   = O(N^{-\frac{1}{2\kf}})$. The strategy is as follows. The iterates are of the form
	\begin{equation}\label{eq:AMP_iterates_Fisher}
		\vect{v^{t + 1}} = \frac{\rdmmat{\rdmmat{S}_{k_F}}}{\sqrt{N}} \vect{\hat v^t} - b_t \vect{\hat v^{t - 1}} = \frac{\rdmmat{Y}_{\kf}}{\sqrt{N}} \vect{\hat v^k} - b_t \vect{\hat v^{t - 1}} + \rdmmat{E} \vect{\hat v^t}
	\end{equation}
	where
	\begin{equation}\label{eq:Fisher_lin_spiked}
		\frac{\rdmmat{Y}_{\kf}}{\sqrt{N}} = \frac{\rdmmat{Z}}{\sqrt{N}} - \frac{1}{\Delta_{k_F}} \mat{W}_{\kf}
	\end{equation}
	is the usual spiked Wigner model for which we know behavior of the AMP iterations. Since $\normop{\rdmmat{E}}   = O(N^{-\frac{1}{2\kf}})$ we have
	\[
	\|  \rdmmat{E} \vect{\hat v^k} \|_2 \leq O(N^{\frac{1}{2\kf}}) \| \vect{\hat v^k} \|_2. 
	\]
	We define the classical spiked iterates by
	\[
	\vect{u^{t + 1}} = \frac{\rdmmat{Y}_{\kf}}{\sqrt{N}} \vect{\hat u^t} - b_t \vect{\hat u^{t - 1}} 
	\]
	which are coupled with $\vect{v}^t$ through the same randomness in $\rdmmat{Y}_{\kf}$ and initialization. We will show that  almost surely for every fixed $t$
	\[
	\| \vect{v^{t}} - \vect{u^{t}} \|_2 \leq C_t  N^{\frac{1}{2\kf}}
	\]
	where the deterministic constant $C_t$ depends on $t$ only. Combined with the fact that the test functions are Lipschitz will imply that almost surely
	\[
	\sup_{\psi \in  \mathrm{PL}_{k + 2}(2,1)} \bigg| \frac{1}{N} \sum_{i = 1}^N \psi(v_i^0, \dots, v_i^t, x_i) - \frac{1}{N} \sum_{i = 1}^N \psi(u_i^0, \dots, u_i^t, x_i)  \bigg| \leq C_t  N^{\frac{1}{2\kf}}
	\]
	from which we can apply the triangle inequality for the Wasserstein~$2$ norm and the usual state evolution to finish the proof. 
	
	It remains to show
	\[
	\| \vect{v^{t}} - \vect{u^{t}} \|_2 \leq C_k  N^{\frac{1}{2\kf}}.
	\]
	We can do an inductive proof. The base case $t=1$ is trivial since both $u$ and $v$ share the same initialization. For each step of the induction, we assume that
	\[
	\| \vect{v^{t}} - \vect{u^{t}} \|_2 \leq C_t  N^{\frac{1}{2\kf}}.
	\]
	to show that 
	\[
	\| \vect{v^{t + 1}} - \vect{u^{t + 1}} \|_2 \leq C_{t + 1}  N^{\frac{1}{2\kf}}.
	\]
	We have
	\begin{align*}
		\| \vect{v^{t + 1}} - \vect{u^{t + 1}} \|_2 &= \bigg\| \frac{\rdmmat{Y}_{\kf}}{\sqrt{N}} \vect{\hat v^t} - b_t \vect{\hat v^{t - 1}} + \rdmmat{E} \vect{\hat v^t} - \frac{\rdmmat{Y}_{\kf}}{\sqrt{N}} \vect{\hat u^t} - b_t \vect{\hat u^{t - 1}} \bigg\|
		\\&\leq \normop{ \frac{\rdmmat{Y}_{\kf}}{\sqrt{N}} } \| \vect{\hat v^t} - \vect{\hat u^t}  \| + |b_t| \| \vect{\hat v^{t - 1}} -\vect{\hat u^{t - 1}}  \| + \normop{\rdmmat{E}} \| \vect{\hat v^t} \| \leq C_{t + 1}  N^{\frac{1}{2\kf}}
	\end{align*}
	since the first two objects are small by the inductive hypothesis and the fact that $ \normop{ \frac{\rdmmat{Y}_{\kf}}{\sqrt{N}} }$ is almost surely bounded, and $\normop{\rdmmat{E}} \| \vect{\hat v^k} \| $ contributes an error of order $O(N^{\frac{1}{2\kf}})$ since $\| \vect{\hat v^t} \| $ is almost surely bounded. We have shown that
	\begin{theo}[State Evolution for the Fisher Matrix]\label{theo:stateevo_fisher}
		The AMP iterates \eqref{eq:AMP_iterates_Fisher} with respect to the Fisher matrix $\rdmmat{S}_\kf$ satisfy the same state evolution as the classical spiked Wigner models with respect to $\rdmmat{Y}$ given by \eqref{eq:Fisher_lin_spiked}. That is, for each $t \in N$, we have
		\[
		\sup_{\psi \in  \mathrm{PL}_{t + 2}(2,1)} \bigg| \frac{1}{N} \sum_{i = 1}^N \psi(v_i^0, \dots, v_i^t, x^\kf_i) - \E( \psi( \mu_0 V + \sigma_0 U, \mu_1 V + \sigma_1 G_1 ,  \dots, \mu_t V + \sigma_t G_k, V ) ) \bigg| \to 0
		\]
		where $V \sim \pi$,  $(\sigma_1 G_1, \dots, \sigma_{t - 1}G_{k-1}) \sim N(0, \Sigma)$ and independent of $U,V$ in Hypothesis $M(1)$ and
		\[
		\Sigma_{k,l} = \begin{cases}
			\E g_{0}(\mu_{0}V + \sigma_{0} U) g_{k-1}(\mu_{k-1}V + \sigma_{k-1} G)	& l = 1\\
			\E g_{l-1}(\mu_{l-1}V + \sigma_{l-1} G) g_{k-1}(\mu_{k-1}V + \sigma_{k-1} G)	& l \geq 2.
		\end{cases}
		\]
	\end{theo}
	
	From the state evolution, we instantly recover the optimal threshold for the Bayesian desnoising functions.
	\begin{cor}[Optimality of the AMP Algorithm for the Fisher Matrix]\label{cor:optimality_spiked}
		Consider the \eqref{eq:AMP_iterates_Fisher} with respect to the Fisher matrix $\rdmmat{S}_\kf$. Let the denoising functions $g_t = \eta$ be the Bayesian optimal ones defined in \eqref{eq:Gauss_denoise}. Then the limiting overlap $\lim_{t \to \infty} \langle \vect{v}^t, \vect{x^\kf} \rangle$ satisfies the following fixed point equation
		\begin{equation} \label{eqn:fixed point_fisher}
			\mu = \E_{x \sim \pi_{\kf},G \sim \mathsf{N}(0,1)} \, \left[ \eta \left( \frac{\mu}{\Delta_\kf}, \frac{\mu}{\Delta_{\kf}}  x + \sqrt{\frac{\mu}{\Delta_\kf} } G \right) x \right].
		\end{equation}
		Furthermore, if $\mu_\star$ is the fixed point of \eqref{eqn:fixed point_fisher}, then
		\begin{align}
			\lim_{t \to \infty}\lim_{N \to \infty}\overline{\mathrm{MSE}}_{N}( \vect{v}^t , \rdmvect{x}) = 1 - \left(\frac{\mathrm{q}_{\star}}{m_{2 \kf}} \right)^2 
		\end{align}
		which matches the IT optimal MSE for spiked Wigner models \eqref{eq:critpiontcond}.
	\end{cor}
	\begin{proof}
		This follows directly from Corollary~\ref{cor:optimality_spikedwigner} with the prior replaced by $\pi_\kf$. The term $m_{2\kf}$ in the MSE comes from the fact that the signal is now $\vect{x}^\kf$, so the variance of $V$ has to be normalized by $m_{2\kf}$ to reduce it to the setting of  Corollary~\ref{cor:optimality_spiked}. 
	\end{proof}
	
	\subsection{A Spectral Method Derived from AMP}
	\label{sec:Linearized_AMP}
	
	We now show that a linearization of the AMP leads to a spectral algorithm with a phase transition for weak recovery that matches the IT threshold. This follows the lines of arguments first described in \cite{lesieur2017constrained}, see also \cite{Mondelli_spectralamp, Mondelli_spectralamp2} and \cite[Section~3]{feng2022unifying}. Consider the linear denoising functions
	\begin{equation}\label{eq:lin_denoiser}
		g_t(x) = \frac{ x }{\sqrt{1 + \mu_t^2}} \qquad\text{and}\qquad \mu_{t + 1} = \frac{m_{2\kf}}{\sqrt{\Delta_{\kf}} \sqrt{1 + \mu_t^{-2}}}.
	\end{equation}
	These denoising functions have an extra $m_{2\kf}$, since we are not necessarily fixing the variance of $\vect{x^\kf}$ in our setting, and the signal to noise ratio is given by $\frac{1}{\sqrt{\Delta_\kf}}$ because the Fisher matrix $\rdmmat{S}_\kf$ has variance $\frac{1}{\Delta_\kf}$.
	Due to the simple linear structure of $g_t$, the AMP algorithm corresponds to matrix multiplication, 
	\[
	\sqrt{1 + \mu_t^2} \vect{\hat v}^{t + 1}  = \frac{\rdmmat{S_\kf}}{\sqrt{N}} \vect{\hat v}^t  - \sqrt{1 + \mu_{t-1}^2} \vect{\hat v}^{t - 1}.
	\]
	This choice of denoising functions with correlated initialization $\vect{v}_0 = \mu_0 \vect{x} + G$, where $G \sim N(0,1)$ asymptotically approximate the power iteration method to find the top eigenvector, see for example \cite[Section~3.3]{feng2022unifying}.
	
	This choice of denoising functions \eqref{eq:lin_denoiser} preserves the variance in the sense that the state evolution equations \eqref{eq:state} satisfies $\sigma_k^2 = 1$ for all $k$, so we only have to track the $\mu$ parameters. By an application of the Banach fixed point theorem, we see that the fixed point equation
	\[
	\mu = \frac{m_{2\kf}}{\sqrt{\Delta_{\kf}} \sqrt{1 + \mu^{-2}}}
	\]
	has a  unique solution at $0$ when the effective SNR of the Fisher matrix satisfies $\frac{1}{\Delta_\kf} < \frac{1}{m^2_{2\kf}}$. When $\frac{1}{\Delta_{\kf}} > \frac{1}{m^2_{2\kf}}$, there exists a non-trivial fixed point at
	\[
	\mu = \sqrt{\frac{m_{2\kf}^2}{\Delta_{\kf}} - 1}.
	\]
	This is the optimal phase transition and matches the transition in \eqref{cor:recovery}. However, we do not have guarantees that the iterations will achieve the non-zero fixed point, nor do we necessarily know if the non-trivial fixed point is the correct maximizer of the variational problem \eqref{eq:RSformula_2}. 
	
	As a consequence of Theorem~\ref{theo:stateevo_fisher} it shows that for $\frac{1}{\Delta_{\kf}} > \frac{1}{m^2_{2\kf}}$, we have
	\[
	\lim_{t\to \infty} \lim_{N \to \infty} \frac{1}{\Delta_{\kf}} \langle \vect{v}^t,\vect{x}^\kf \rangle \to \mu = \sqrt{\frac{m^2_{2\kf}}{\Delta_{\kf}} - 1}
	\]
	which implies that
	\[
	\lim_{t\to \infty} \lim_{N \to \infty} \frac{ \langle \vect{v}^t,\vect{x}^\kf \rangle }{m_{2\kf}} \to \sqrt{1 - \frac{\Delta_{\kf}}{m_{2\kf}^2}} 
	\]
	which coincides with the classical BBP transition for spiked matrices and $\mathrm{q}_{0}$ in \eqref{eq:def_limoverlap_PCA_Fisher}. We conclude the following.
	
	\begin{cor}\label{cor:BBPforLinAMP}
		Under the same setting of Thm.~\ref{th:informal_rank1decompositionFisher}, if $\rdmvect{v}_1$ the top eigenvector of  $\rdmmat{S}_\kf$, $g_t$ is the normalized linear function defined by Eq.~\eqref{eq:lin_denoiser} and $\rdmvect{\hat{x}}_{t} \hookleftarrow \mathtt{AMP}_t( \frac{\rdmmat{S}_\kf}{\sqrt{N}} ,f, \Delta_\kf ;\rdmvect{v}_1)$, then we have:
		\begin{align}
			\lim_{t \to \infty} \lim_{N \to \infty} \overline{\mathrm{MSE}}_{N} \left( \vect{x}^{\kf} , \rdmvect{\hat{x}}_{t} \right) > 0\, ,  
		\end{align}
		when $\frac{1}{\Delta_{\kf}} > \frac{1}{m^2_{2\kf}}$.
	\end{cor}
	
	
\end{document}